\documentclass{article} 
\usepackage{arxiv,times}

\usepackage{amsmath,amsfonts,bm}

\def\eqref#1{equation~\ref{#1}}

\def\1{\bm{1}}

\DeclareMathAlphabet{\mathsfit}{\encodingdefault}{\sfdefault}{m}{sl}
\SetMathAlphabet{\mathsfit}{bold}{\encodingdefault}{\sfdefault}{bx}{n}

\newcommand{\E}{\mathbb{E}}

\newcommand{\R}{\mathbb{R}}

\usepackage{placeins}
\usepackage{hyperref}
\usepackage{url}
\usepackage{amsmath, amssymb, amsthm, amsfonts}
\usepackage{algorithm}
\usepackage{algpseudocode}
\usepackage{graphicx}
\usepackage{xcolor}
\usepackage{bm}
\usepackage{hyperref}
\usepackage{subfig}
\usepackage{cite}
\usepackage{mathtools}
\usepackage{cleveref}
\usepackage{commath}
\usepackage{wrapfig}

\algrenewcommand\algorithmicrequire{\textbf{Input:}}
\algrenewcommand\algorithmicensure{\textbf{Output:}}

\newcommand{\var}{\text{var}}
\newcommand{\p}{\mathbb{P}}
\newcommand{\e}{\mathbb{E}}
\newcommand{\cov}{\text{cov}}
\newcommand{\fracc}[1]{\frac{1}{{#1}}}
\newcommand{\z}{\vec{z}}
\newcommand{\tr}{\text{tr}}
\newcommand{\bn}{\vec{B}} 
\newcommand{\bni}[1]{B_{#1}} 
\newcommand{\bnilower}[1]{b_{#1}} 

\theoremstyle{plain}
\newtheorem{theorem}{Theorem}[section]
\newtheorem{lemma}{Lemma}[section]

\theoremstyle{definition}

\newtheorem{notation}[lemma]{Notation}

\crefname{theorem}{Theorem}{Theorems}
\crefname{lemma}{Lemma}{Lemmas}
\crefname{proposition}{Proposition}{Propositions}
\crefname{definition}{Definition}{Definitions}
\crefname{remark}{Remark}{Remarks}

\title{Boundary Detection for High-Dimensional Point Clouds}

\author{Abigail Hickok\\
Department of Mathematics\\
Columbia University\\
New York, NY, USA \\
\texttt{ah3966@columbia.edu} \\
\And
Valerio Galanti \\
Department of Computer Science\\
Columbia University\\
New York, NY, USA \\
\texttt{vg2551@columbia.edu}
\AND
Osama Al Dalahmah \\
Department of Pathology and Cell Biology\\
Columbia University\\
New York, NY, USA \\
\texttt{oa2298@cumc.columbia.edu}
\And
Andrew J. Blumberg\\
Department of Mathematics,\\
Department of Computer Science,\\
Irving Institute for Cancer Dynamics \\
Columbia University\\
New York, NY, USA \\
\texttt{andrew.blumberg@columbia.edu}
}

\iclrfinalcopy 
\begin{document}

\maketitle

\begin{abstract}
We introduce two new methods for boundary detection in high-dimensional point cloud data, intended to work in situations with non-uniform density, high ambient dimension, and substantial noise. Our first method fits a separating hyperplane on local neighborhoods after doing local dimension reduction---this method works best when one can reliably estimate the intrinsic dimension of the data. Our second method measures the directional uniformity of local neighborhoods, and has the best overall performance on synthetic data. In addition to pointwise numerical boundary scores, our algorithms yield an estimated normal vector at each boundary point. We validate our methodology on synthetic data as well as two real examples: we analyze spatial transcriptomics data, where we show that we can detect the boundary of a tumor, and we also study LLM activation geometry.
\end{abstract}

\section{Introduction}

The manifold hypothesis specifies that point cloud data is sampled from a distribution supported on a low-dimensional manifold.  A natural variant of this setup is to assume that the support is instead a manifold with boundary.  In this case, determining the location of the boundary and a normal vector to it become fundamental questions.  Algorithms to determine the boundary have many applications.
For example, in spatial transcriptomics data, each point on a grid has a gene expression vector that lies in $\mathbb{R}^n$ for $n$ approximately 20000.  Detecting boundary regions in genomic space allow us to determine qualitative changes in underlying tissue, for example transitions from cancerous tumor regions to healthy cells. Another application of interest is to study the geometry of LLM activations. Recent mechanistic interpretability work has discovered that residual-stream activations often lie on curved submanifolds rather than linear subspaces~\citep{gurnee2023, engels2024, gurnee2026, berdnikov2026}. Model behavior can be influenced by adding a ``steering'' vector at intermediate activations \citep{turner2023, panickssery2023steering}, but these interventions work best when one is able to stay on the underlying manifold \citep{wurgaft2026manifold}. Knowing the boundary points as well as having estimated boundary normals is useful for steering interventions and other interpretability problems. More broadly, boundary detection can be used for general-purpose methods such as clustering \citep{gravitational, border_peeling_2020}.

\subsection{Contributions}

We introduce two general-purpose methods for identifying boundary points and calculating normal vectors at the boundary. Both methods are fast and practical for large data sets.

Our hyperplane method works the best among all tested methods when the local dimension can be reliably estimated. When the estimated dimension $m$ is larger than the true dimension, it finds the correct boundary of the data as though it were $m$-dimensional. Our directional uniformity method is more robust when noise levels prevent accurate dimension estimates.

The hyperplane method is mathematically grounded by theorems proving that boundary points can be reliably distinguished from interior points in a way that scales well with intrinsic dimension. We empirically validate both methods on synthetic data.

Finally, we show that our methodology is effective on spatial transcriptomics data for identifying a tumor boundary, and we also apply it to find the boundary of LLM activations for the character-counting manifold of \citep{gurnee2026} for open-weight models. 

We have released the source code and scripts for reproducing our figures and experiments on github: \url{https://github.com/aghickok/boundary_detection}.

\subsection{Related Work}
There are several boundary-detection methods that rely on the fact that boundary points have lower density than interior points. BORDER \citep{xia_border_2006} counts the number of ``reverse $k$ nearest neighbors'' (R$k$NN), i.e., the number of points for which the query point is a $k$NN. Boundary points tend to have fewer $k$NNs. \citet{border_peeling_2020} and \citet{border_peeling_2021} also use reverse $k$NNs for boundary detection, as an intermediate step towards clustering. In more detail, \citet{border_peeling_2020} and \citet{border_peeling_2021} calculate $\sum_{x_j \in \text{RkNN}(x_i)} K(x_i, x_j)$, where $K$ is a Gaussian or Cauchy kernel, respectively. Reverse unreachability (RU) \citep{reverse_unreachability_2018} is different; it uses the representation matrix $W$ from \citet{LLE} to approximate each point $x_i$ as a linear combination $\sum_j W_{ij} x_j$ of its neighbors, and then reverse unreachability is defined as $\sum_j \mathbf{1}_{W_{ji} \leq 0}$. All of these methods perform poorly when the point cloud has non-uniform density.

Some methods are only designed to be used for very low-dimensional data. The direction centrality measure \citep{dcm} is defined for data in $\R^2$ and the level-set method \citep{LSM_2015} was only applied to $2$D data. Both methods can be theoretically generalized to higher dimensions, but scale poorly with ambient dimension and are not practical for data that is more than two dimensional.

The last class of methods uses geometric structure to detection boundaries in data of any dimension. Border Edge Pattern Selection (BEPS) \citep{beps}, CO \citep{gravitational}, and LoDD \citep{lodd} all start by projecting the nearest neighbors $\{x_{ij}\}$ to the unit sphere centered at the query point $x_i$. Let $\vec{y}_{ij} := (x_{ij} - x_i)/\norm{x_{ij} - x_i}$ and define $\vec{v_i} := \fracc{k} \sum_{j=1}^k \vec{y}_{ij}$. BEPS is similar to our hyperplane method---the BEPS score $\fracc{k}\sum_j \mathbf{1}_{y_{ij} \cdot \vec{v_i} \geq 0}$ is the percentage of neighbors on the positive side of the hyperplane orthogonal to $\vec{v_i}$. It is far less robust to noise than either of our methods, though, because the direction of the vector $\vec{v_i}$ is sensitive to noisy perturbations off the manifold, especially on interior points where $\e[\vec{v_i}] = 0$.

The CO score \citep{gravitational} uses the same vector $\vec{v_i}$ but scales by density and measures the alignment between neighbors. At query point $x_i$, the score is $\sum_{j=1}^k m_i^{-1} m_j^{-1} \vec{v_i} \cdot \vec{v_j}$, where $m_i^{-1} = \sum_j \norm{x_{ij} - x_i}$ is a ``mass'' term that is high in regions with low density. Like the other density-based methods, CO performs poorly on data with non-uniform density.

The LoDD score is $\omega (\sum_{i=1}^d \lambda_i)^2 + (1 - \omega) \frac{2d}{d-1}\sum_i \sum_{j \neq i} \lambda_i \lambda_j$, where $\{\lambda_i\}$ is the set of eigenvalues of the covariance matrix for the centered, normalized neighbor vectors $\{ \vec{y}_{ij}\}$ and $\omega \in [0, 1]$ is a hyperparameter that is set to $.5$ in \citet{lodd} and in our experiments in this paper. The LoDD score has a maximum of $1$ that is achieved when the eigenvalues are equal and sum to $1$. It is higher for interior points because $\{ \vec{y}_{ij}\}$ is uniformly distributed on the $d$-dimensional sphere. LoDD has the strongest overall performance of all the existing methods we tested in the literature, but it is less robust than ours to noise, especially in high ambient dimensions. Unlike our methods (or BEPS or CO), LoDD does not produce normal vectors at each point---just pointwise boundary scores.

\section{Hyperplane method for boundary detection}

\begin{figure}
    \centering
    \vspace{-8mm}
    \subfloat[]{\includegraphics[width = .2\textwidth]{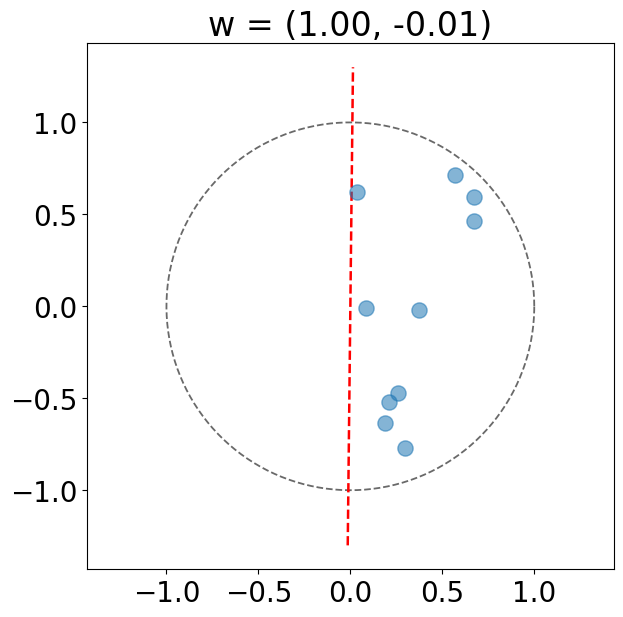}}
    \subfloat[]{\includegraphics[width = .2\textwidth]{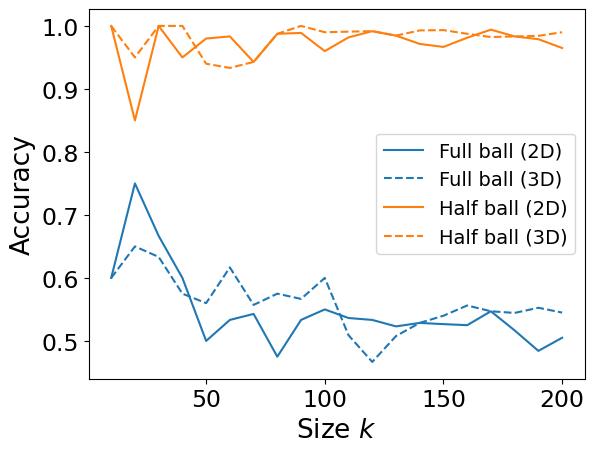}}
    \subfloat[]{\includegraphics[width = .2\textwidth]{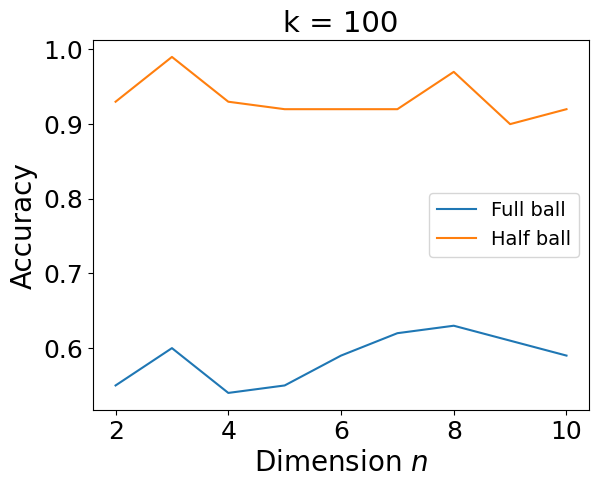}}
    \caption{Hyperplane optimization on points sampled from half and full balls. Accuracy is the percentage of points that lie on the positive side of the hyperplane orthogonal to $\vec{w}$. 
}
    \label{fig:boundary_ball_model}
\end{figure}

The idea of our hyperplane (HP) boundary-detection algorithm, detailed in \Cref{alg:boundary_hyperplane}, is that if $x$ is a boundary point on an $n$-dimensional manifold, then it has a neighborhood that is diffeomorphic to an $n$-dimensional half ball. A half space is defined by a hyperplane in $\mathbb{R}^n$. If $x$ is a boundary point, then it is possible to find a hyperplane such that most its neighbors lie on the same side of the hyperplane. On the other hand, if $x$ is an interior point, then any hyperplane will split the neighborhood in half, with roughly half the points on one side and half on the other side. 

We assign each point $x$ a \emph{boundary score} $s(x) \in [0, 1]$ that is the percentage of its neighborhood lying on the same side of an ``optimal'' hyperplane. Boundary points have scores near $1$, and manifold points have scores near $.5$. Near a transverse intersection of manifolds, the score is usually higher than $.5$ but lower than $1$---for example, near a transverse intersection of two manifolds, the score is near $.75$. (See \Cref{fig:benchmark_boundary_dimension_failure}, for example.)\footnote{For example, suppose a point cloud is sampled from the intersection of the $xy$ plane and the $xz$ plane in $\R^3$. For a point on the $xy$ plane with small but positive $y$ coordinate, the optimal hyperplane $P$ is one that is parallel to the $xz$ plane with normal vector $\vec{w} = (0, -1, 0)$. The boundary score will be close to $.75$ because about half the points on the $xy$ plane and \emph{all} the points on the $xz$ plane will be on the positive side of $P$.} The algorithm steps are described below.

\subsection{The algorithm} 
\begin{enumerate}
    \item {\bf Local PCA:} For each point $x$, we use local PCA to reduce the dimension of its local neighborhood $\{x_i\}_{i=1}^k$, where $k$ is a parameter that controls the size of the neighborhood. If the dimension $n$ of the manifold is known, then we reduce to $n$ dimensions; otherwise, we use a spectral gap threshold condition. Let $\{\vec{z_i}\}_{i=1}^k$ be the neighborhood after dimension reduction and after centering so that $x$ is mapped to $\vec{0}$.

    \item {\bf Hyperplane optimization:} We solve for an ``optimal'' hyperplane that best separates the data points onto a single side of the hyperplane. Every hyperplane is determined by an orthogonal vector $\vec{w}$ with norm $\norm{\vec{w}} = 1$. A point $\vec{z_i}$ lies on the ``positive'' side of the hyperplane if $\vec{z}_i \cdot \vec{w} > 0$. A hyperplane is good if most of the points lie on the positive side of the hyperplane. To formulate this, we solve the optimization problem
    \begin{equation}\label{eq:boundary_optimization}
        \text{argmax}_{\vec{w}} \, f(\vec{w}) = \sum_{i=1}^k \vec{z_i} \cdot \vec{w}\,,  \qquad
        \text{subject to } g(\vec{w}) = \|\vec{w}\|^2 = 1\,.
\end{equation}
    The function $f(\vec{w})$ is proportional to the mean (signed) distance to the hyperplane, and the constraint normalizes the sum so that it cannot be arbitrarily large. By the method of Lagrange multipliers,
    the solution to our optimization problem is $\vec{w} = \sum_i \vec{z_i}/\norm{\sum_i \vec{z_i}}\,.$

    \item {\bf Boundary score:} The boundary score $s(x) \in [0, 1]$ for point $x$ is the percentage of neighboring points that lie on the positive side of the hyperplane. Boundary and interior points have scores close to $1$ and $.5$, respectively.
\end{enumerate}

As an alternative to our hyperplane optimization step, one could instead aim to maximize the number of points such that $\vec{z_i} \cdot \vec{w} \geq 0$. In contrast, maximizing $f(\vec{w})$ encourages but does not force points to lie on the positive hyperplane side. 
The advantage of our formulation is that it has a fast, closed-form solution. \Cref{alg:boundary_hyperplane} runs in about $1$s for $N = 10,000$ points in $\R^{15}$.

\begin{algorithm}
\caption{Boundary Detection}\label{alg:boundary_hyperplane}
\begin{algorithmic}[1]

\Require Point cloud $\mathcal{X}$, neighborhood size $k$, PCA spectral gap threshold $\tau$ (default $\tau = 7$)
\Ensure Boundary score $s(x) \in [0, 1]$ for each $x \in \mathcal{X}$

\ForAll{$x \in \mathcal{X}$}
    \State $\mathcal{N} \leftarrow k\text{-nearest neighbors of } x$ in $\mathcal{X}$
    \State $\{z_i\} \leftarrow \text{PCA}(\mathcal{N},\, \tau)$ \Comment{Use threshold $\tau$ to determine dimension}
    \State $\{z_i\} \leftarrow \{z_i - z\}$ \Comment{Center so projection of $x$ is at origin}
    \State $\vec{w} \leftarrow \dfrac{\sum_i \vec{z}_i}{\left\|\sum_i \vec{z}_i\right\|}$ \Comment{$\vec{w}$ is orthogonal to an optimal hyperplane}
    \State $m \leftarrow \bigl|\{z_i \mid \vec{z}_i \cdot \vec{w} \geq 0\}\bigr|$ \Comment{Count neighbors on positive side of the hyperplane}
    \State $s(x) \leftarrow m \;/\; k$ \Comment{Divide by total size of the neighborhood}
\EndFor
\State \Return $\{s(x)\}_{x \in \mathcal{X}}$
\end{algorithmic}
\end{algorithm}

To illustrate the hyperplane optimization step of the algorithm, we sample $k$ points $\{\vec{z_i}\}$ from half or full $n$-dimensional balls and calculate $\vec{w}$. See \Cref{fig:boundary_ball_model}. For each experiment, we define the ``accuracy'' to be the percentage of points that lie on the positive side of the hyperplane determined by $\vec{w}$---this is analogous to the boundary score defined previously. When the points are sampled from a half ball---which models the boundary-point case---the accuracy stays close to $1$, even for the smallest neighborhood size that we tested ($k = 10$). By contrast, when the points are sampled from a full ball---which models the interior-point case---the accuracy is close to $1/2$. The gap between interior and boundary points persists as the dimension $n$ increases to $n = 10$. This is preliminary evidence that the optimized hyperplane reliably distinguishes between boundary and interior points.

\subsection{Concentration Inequalities for the Boundary Scores}\label{sec:math}

The theorems below can be used to guide a choice of neighborhood size $k$ given a manifold dimension $n$, or to pick a threshold on the boundary scores for identifying boundary points. Our goal is to quantify the distribution of boundary scores $s(x)$ when $x$ is either an interior point or a boundary point. To do so, we consider the random variable
\begin{equation}\label{eq:snk}
    S_{nk} := \sum_{i=1}^k \mathbf{1}_{\vec{w} \cdot \vec{z_i} \geq 0}\,, \qquad \qquad 
\end{equation}
which counts the number of points that lie on the positive side of the hyperplane determined by $\vec{w}$. In our theoretical results below, we assume that the points $\{\vec{z_i}\}_{i=1}^k$ are drawn uniformly at random from either an $n$-dimensional ball (to model the situation where $x$ is an interior point) or an $n$-dimensional half ball (to model the situation where $x$ is a boundary point). The boundary score $s(x)$ behaves similarly to $\frac{1}{k}S_{nk}$, and we are interested in its probability distribution. The challenge is that the terms in the sum $S_{nk}$ are generally not independent because $\vec{w}$ depends on all $\vec{z_i}$, so $S_{nk}$ is not a binomial random variable. Below, we prove that boundary scores for boundary and interior points concentrate near $1$ and $1/2$, respectively.

\begin{theorem}\label{thm:full}
    Let $\{\vec{z_i}\}_{i=1}^k$ be points sampled independently and uniformly at random from an $n$-dimensional ball. Then there are constants $C_1, C_2$ such that for any $\epsilon > 0$, we have
    \begin{equation*}
        \p\Big[ \vert S_{nk}/k - 1/2 \vert > \epsilon \Big] \leq \epsilon^{-2} \Big( C_1 \sqrt{n/k} + C_2 n/k\Big)
    \end{equation*}
   for all dimensions $n$ and all $k \geq 3$.
\end{theorem}

\begin{theorem}\label{thm:half}
    Let $\{\vec{z_i}\}_{i=1}^k$ be points sampled independently and uniformly at random from an $n$-dimensional half ball. For any $\epsilon > 0$, we have
    \begin{equation*}
        \p\Big[ 1 - S_{nk}/k \geq \epsilon \Big] \leq \frac{3}{\sqrt[3]{4}} \, \frac{(n+1)}{\epsilon \sqrt[3]{k}}\,.
    \end{equation*}
\end{theorem}

\section{Directional Uniformity Method for Boundary Detection}\label{sec:neighborhood_uniformity}

\begin{figure}
    \centering
    \vspace{-5mm}
    \includegraphics[width = .75\textwidth]{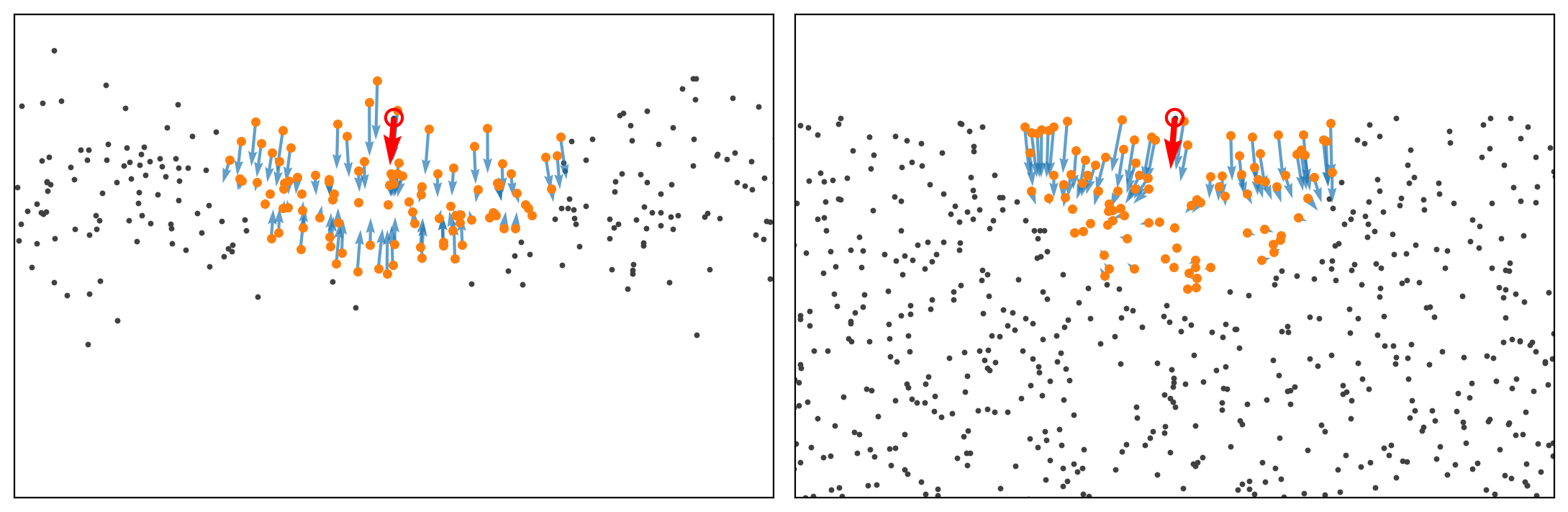}
    \caption{The vectors $\vec{v_j}$ for points $\{x_{ij}\}$ (orange) in a local neighborhood  of a point $x_i$ (red) within a larger point cloud that is sampled from (left) a thin strip in $\mathbb{R}^2$ with noise and (right) a half plane.}
    \label{fig:thin_boundary}
\end{figure}
Our directional uniformity (DU) method is related to the Rayleigh uniformity test for deciding whether a spherical sample is uniform or not \citep{mardia_directional_1999}. Suppose that for each point $x$, we consider a local neighborhood $\{x_{ij}\}$ of size $k$. We project each neighbor onto a sphere centered at $x$ and calculate the vector $\vec{v_i} := \fracc{k}\sum_j (x_j - x_i)/\norm{x_j - x_i}$. If $x$ is a manifold point, then the neighbors are uniformly distributed on the sphere, so $\E[\vec{v_i}] =0$. If $x$ is a boundary point, then $\E[\vec{v_i}]$ is nonzero and parallel to the boundary normal. We define the boundary score
\begin{equation}\label{eq:uniformity_method}
    s(x) := \fracc{k^3}\sum_j \vec{v_i} \cdot \vec{v_j}\,.
\end{equation}
Note that expected boundary point scores are strictly less than $1$ and depend on dimension $n$.

\begin{algorithm}
\caption{Directional-Uniformity (DU) Boundary Detection}\label{alg:boundary_neighborhood_uniformity}
\begin{algorithmic}[1]

\Require Point cloud $\mathcal{X}$, neighborhood size $k$
\Ensure Boundary score $s(x) \in [0, 1]$ for each $x \in \mathcal{X}$

\ForAll{$x \in \mathcal{X}$}
    \State $\mathcal{N} \leftarrow k\text{-nearest neighbors of } x$ in $\mathcal{X}$
    \State $\vec{v_i} \leftarrow \sum_{j \in \mathcal{N}} \frac{x_j - x_i}{\norm{x_j - x_i}}$
    \State $s(x) \leftarrow \fracc{k^3}\sum_j \vec{v_i} \cdot \vec{v_j}$ 
\EndFor
\State \Return $\{s(x)\}_{x \in \mathcal{X}}$
\end{algorithmic}
\end{algorithm}

Robustness-to-noise is the main advantage of \Cref{eq:uniformity_method} instead of simply $\norm{\vec{v_i}}$. The former equation measures the alignment of $\vec{v_i}$ with its neighbors $\vec{v_j}$. Although there is always correlation between $\vec{v_i}$ and $\vec{v_j}$ due to the overlap in neighborhoods, there is stronger alignment when $x$ is a boundary point because $\E[\vec{v_i}]$ is nonzero.

More importantly, ambient noise can thicken an $n$-dimensional manifold into a higher-dimensional object, and it is challenging for boundary detection methods to distinguish between the boundary of the ground-truth $n$-dimensional manifold and the boundary of the higher-dimensional object;  \Cref{eq:uniformity_method} helps to mitigate that. For example, in \Cref{fig:thin_boundary}, a 1D horizontal line segment in $\R^2$ has noise added so that it is a thin $2$D subset. A point $x_i$ at the top (or bottom) of the subset has a neighborhood that is genuinely a half ball, so $\norm{\vec{v_i}}$ is as large as it is for a ``true'' boundary point. However, because it is a thin 2D strip, the points $x_j$ at the bottom of the neighborhood are boundary points on the \emph{opposite} side. For points $x_j$ at the top and bottom, the vectors $\vec{v_j}$ point down and up, respectively. Therefore, there are cancellations in \Cref{eq:uniformity_method} so that it is small, even though most of the vectors $\vec{v_j}$ are large. This correctly signifies that the point $x_i$ is not on the true boundary (the left and right endpoints of the line segment). For comparison, the right panel of \Cref{fig:thin_boundary} shows $\{\vec{v_j}\}$ for points $x_j$ in a neighborhood of a point $x_i$ on the boundary of a half plane. The norm $\vec{v_i}$ is close to that of the point on the line segment, but $\sum_j \vec{v_i} \cdot \vec{v_j}$ is larger. At a fixed neighborhood size and high noise level, \Cref{eq:uniformity_method} is less likely than $\norm{v^i}$ to flag interior points as boundary points.

\section{Boundary Detection on Synthetic Data}\label{sec:boundary_experiments}

\paragraph{Non-uniform density.}\label{sec:density}

Many methods in the literature erroneously detect low-density interior points as boundary points. To test this, we sample points non-uniformly from a 2D disk where there is a low-density region in the interior. For each method, we calculate AUC values (shown in \Cref{tab:density_well_auc}) for boundary points vs. deep interior points (the low-density region).\footnote{The ground-truth disk boundary is the set of points $x$ where $\norm{x} = 1$. This boundary has measure zero, so zero of the sampled points are on the boundary. Instead, we conservatively define the empirical boundary to be the set of points with $\norm{x} > .98$. In later sections, we lower this threshold to $.95$ with justification given in \Cref{sec:auc}, where we also show the results for other threshold choices.} See \Cref{fig:density} for the pointwise scores. RU \citep{reverse_unreachability_2018}, BORDER \citep{xia_border_2006}, and BP \citep{border_peeling_2020, border_peeling_2021} explicitly rely on the idea that boundary points have lower density, so it is unsurprising that they have AUC values below $.5$. The CO method \citep{gravitational} performs similarly poorly (AUC $.447$) because of a ``mass'' term that increases the score at low-density points. The only methods that perform well on the non-uniform-density data set are LoDD, BEPS, our HP method, and our DU method, which all have AUC $= 1$.

\begin{table}[htbp]
  \centering
  \resizebox{\textwidth}{!}{%
  \begin{tabular}{ccccccccc}
    \hline
    DU (us) & Hyperplane (us) & LoDD & BEPS & RU & CO & BORDER & BP (Gauss.) & BP (Cauchy) \\
    \hline
    1.000 & 1.000 & 1.000 & 1.000 & 0.474 & 0.447 & 0.474 & 0.425 & 0.436 \\
    \hline
  \end{tabular}%
  }
  \caption{Boundary ($\norm{x} > .98$) vs deep-interior ($\norm{x} < .23$) AUC for $N = 10,000$ points sampled non-uniformly from a 2D disk. In the interior, there is a ``well'' ($r < .23$) where the density is $1/9$ that of the density in the outer region ($r > .67$). The density increases linearly in between.}
  \label{tab:density_well_auc}
\end{table}

\paragraph{Surfaces with non-trivial geometry.}\label{sec:surfaces}
\begin{figure}
    \centering
    \vspace{-5mm}
    \includegraphics[width = .6\textwidth]{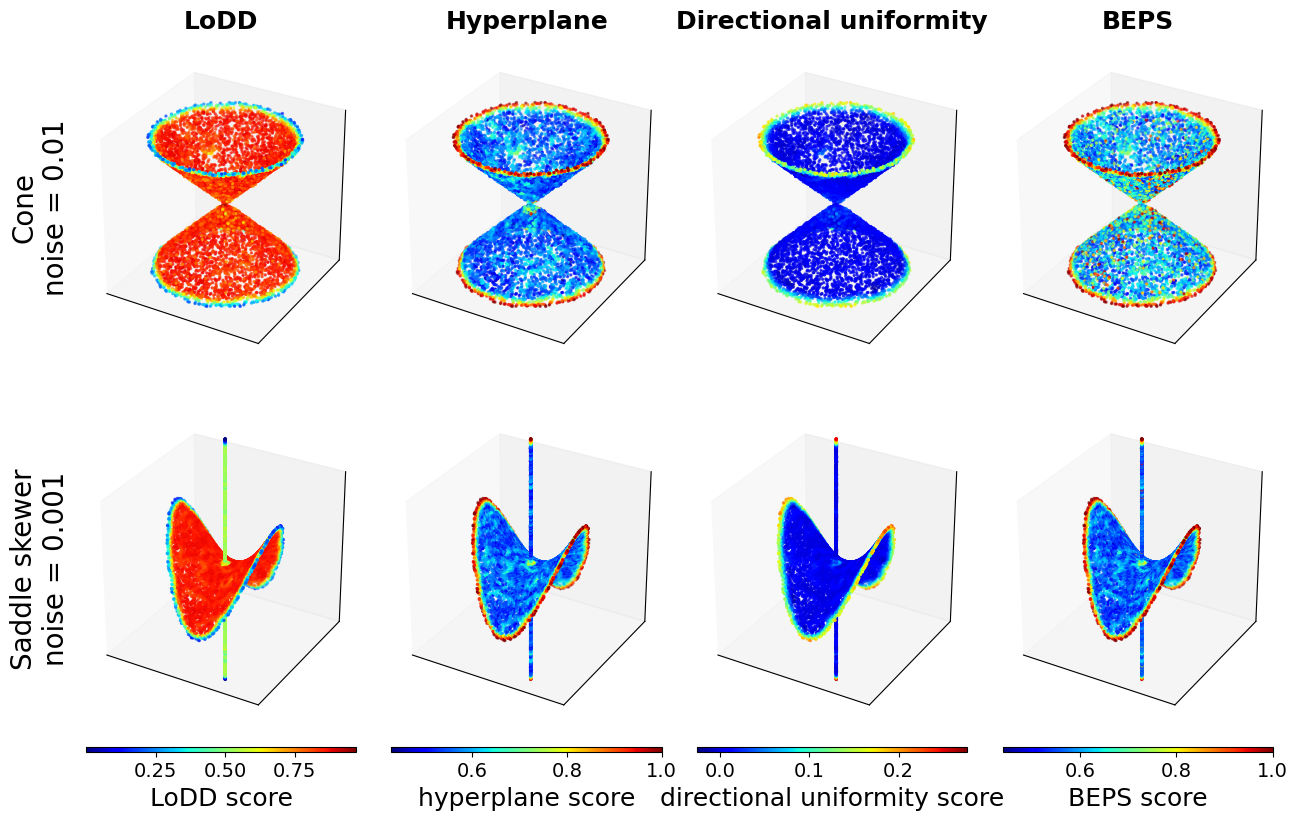}
    \caption{Boundary scores ($k = 100$) for two synthetic data sets ($N = 10^4$) with low noise levels.}
\label{fig:benchmark_boundary_scores}
\end{figure}

For the remainder of \Cref{sec:boundary_experiments}, our primary focus is comparing LoDD \citep{lodd} and BEPS \citep{beps} with our HP and DU methods. We test the four methods on six synthetic data sets with low levels of Gaussian noise; see \Cref{fig:benchmark_boundary_scores} for the cone and saddle-skewer, and \Cref{fig:benchmarks_appendix} for four more. BEPS performs reasonably well, but it has a high number of false positives (interior points that score as high as boundary points) on three of the surfaces (the cone, Saturn, and the three disks). LoDD struggles with data points that have intrinsic dimension one, such as the skewer in the saddle-skewer. The interior 1D points have LoDD scores that are as low as points near the surface boundaries. (LoDD scores are close to $0$ on boundary points and higher on the interior, unlike the other three methods.) HP scores boundary points close to $1$, while manifold points have low boundary scores close to $.5$. Points near (but not on) a transverse intersection of two surfaces have moderate hyperplane boundary scores near $.75$. Similarly, DU scores are higher at boundary points and lower (close to $0$) at interior points. This data set shows that all four methods are reasonably accurate on surfaces with complex topologies and curvature, except that LoDD confuses interior 1D points with boundary points and BEPS is prone to false positives, sometimes confusing isolated interior points for boundary points.

\paragraph{High ambient dimension.}\label{sec:high_ambient}
\begin{figure}
    \centering
    \includegraphics[width = \textwidth]{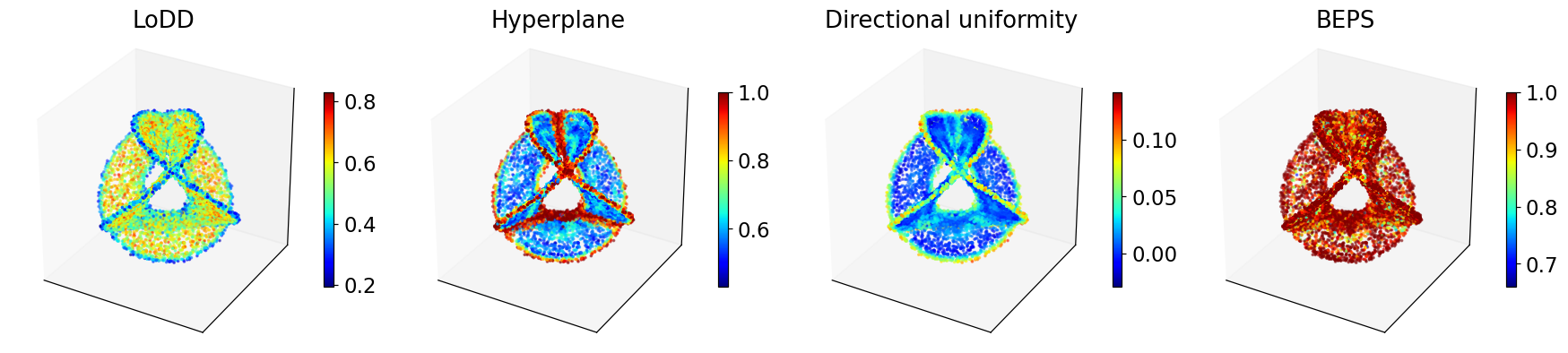}
    \caption{Boundary scores ($N = 10,000$, $k = 100$) on the Henneberg minimal surface embedded in $\mathbb{R}^{15}$ with isotropic Gaussian noise ($\sigma = .02$), plotted in $\mathbb{R}^3$.}
    \label{fig:benchmarks_high_ambient}
\end{figure}
\vspace{-5pt}
In high ambient dimension with low noise, BEPS cannot distinguish between boundary points and interior points, and LoDD scores have high variance on interior points. For example, see  \Cref{fig:benchmarks_high_ambient}. BEPS scores are roughly constant across all points, and many interior points have LoDD scores that are as low as boundary-point LoDD scores. In contrast, our HP and DU methods both highlight a clear boundary, scoring interior points much lower, although HP has a small amount of interior noise. A limitation of HP in high ambient dimensions is that it scores self-intersections as high as boundary points. (In lower ambient dimensions, HP assigns score $\approx .75 < 1$ to points near intersections.) For the remainder of this subsection, we focus on HP, DU, and LoDD.

Next, we study $2$D disks embedded in $\R^D$ ($3 \leq D \leq 15$) with isotropic Gaussian noise ($\sigma = .01$). In \Cref{fig:disks_high_ambient}(a), we calculate AUC values for interior vs. boundary points\footnote{We define the empirical boundary to be the set of points with $\norm{x} > .95$, and the empirical interior to be the set of points with $\norm{x} < .8$. We ignore the points with norms in between because their classification is ambiguous. See \Cref{sec:auc} for more details on how we derived our definitions of interior and boundary points and \Cref{fig:disks_high_ambient_AUC_sweep} for AUC values for alternative choices.}. Both the hyperplane and DU methods have AUC values near $1$ at all $D$, while LoDD's performance degrades as $D$ increases.  That is, HP and DU perform better than LoDD on point clouds with high ambient dimension.

\begin{figure}
    \centering
    \includegraphics[width = .5\textwidth]{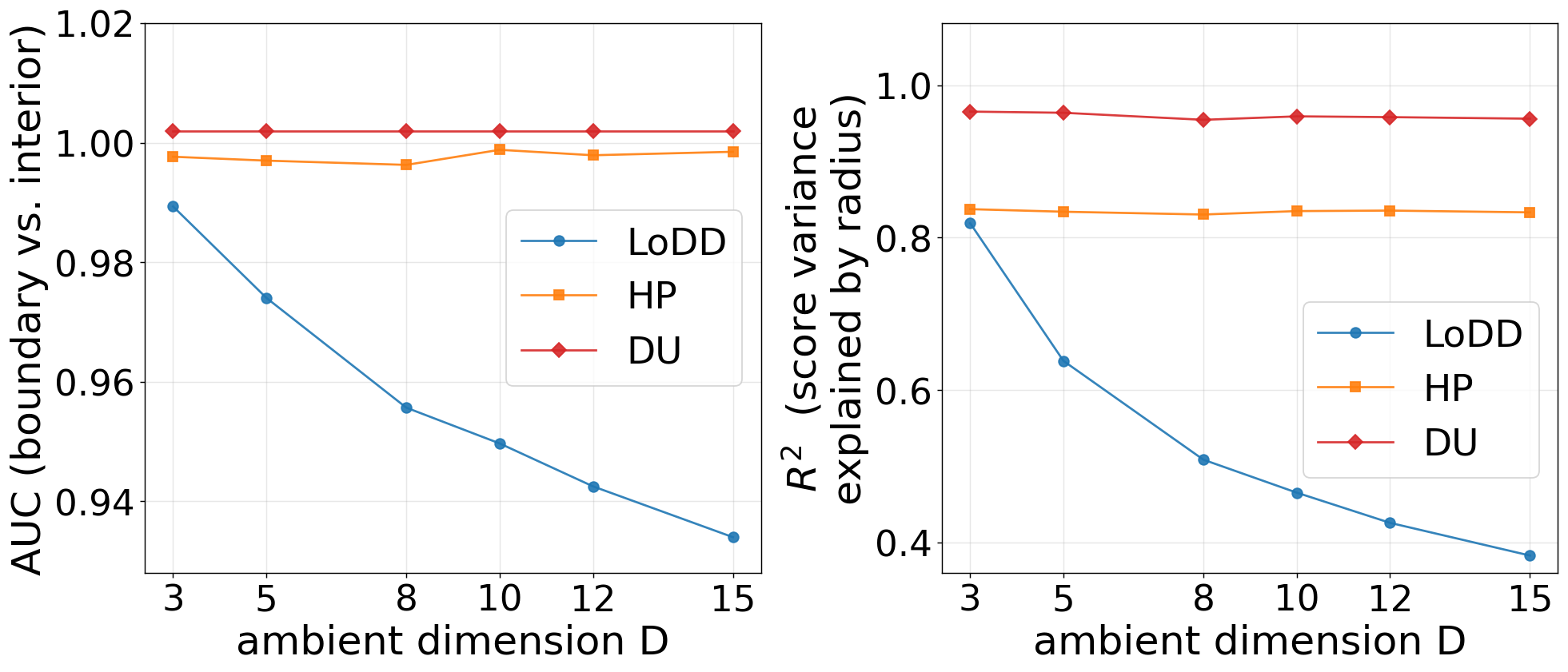}
    \caption{Analysis of boundary scores on disks embedded in $\mathbb{R}^{D}$ with Gaussian noise $\sigma = .01$ (size $N = 10,000$ points, with $k = 100$). 
    }
    \label{fig:disks_high_ambient}
\end{figure}

For each method (excluding BEPS), we measure the strength of the relationship between a point's norm and its boundary score. A good method should score all interior points at an approximately constant value and then monotonically change as one approaches the boundary. For each method we fit a monotonic function $f(r)$ to the scores as a function of radius $r$ and calculate the coefficient of determination $R^2$. \Cref{fig:disks_high_ambient}(b) shows that the DU method has the highest $R^2$ value (near $1$) at each $D$, while LoDD has the lowest values, reflecting its high variance within the interior of the disk.

\paragraph{High noise levels.}\label{sec:high_noise}
\begin{figure}
    \centering
    \includegraphics[width = .7\textwidth]{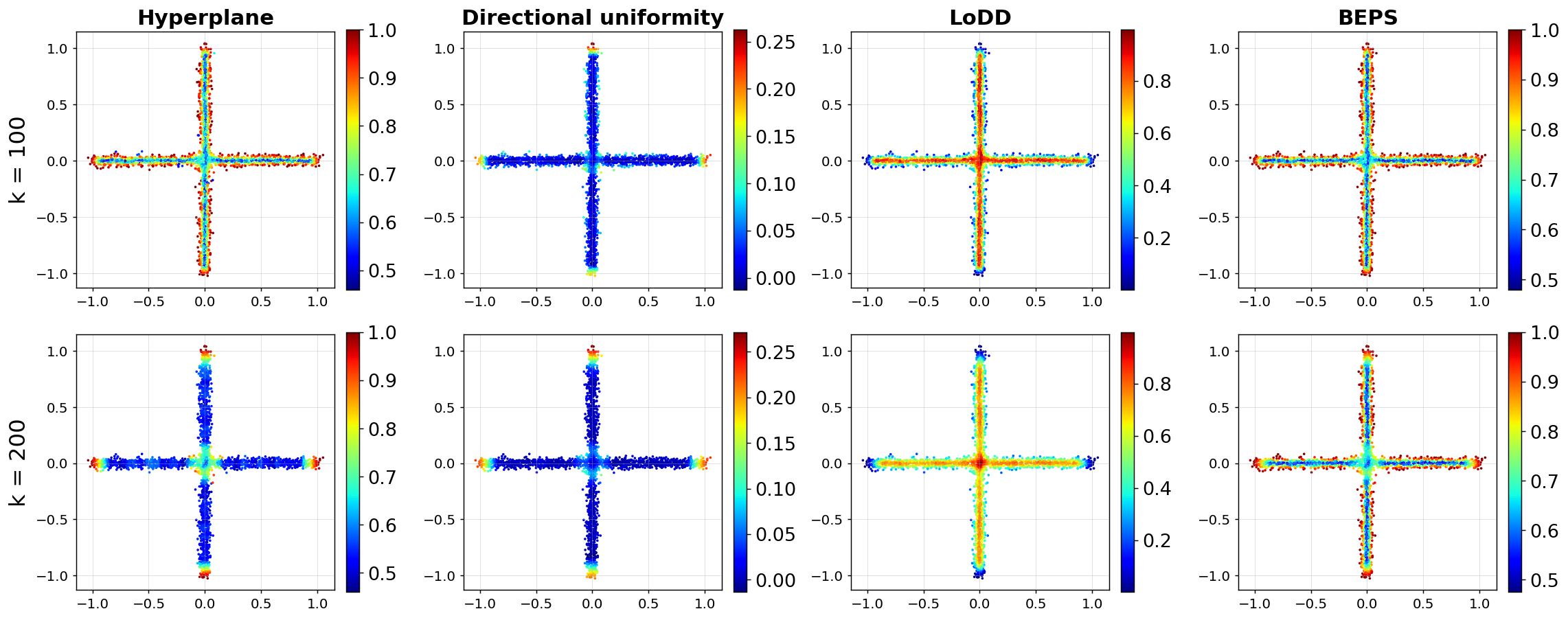}
    \caption{Boundary scores on noisy intersecting line segments.}
\label{fig:benchmark_boundary_dimension_failure}
\end{figure}

At higher levels of noise, local PCA may overestimate the local dimension. This affects HP. If the estimated dimension $m$ is larger than the true manifold dimension $n$, then HP will find the boundary of it \emph{as an $m$-dimensional object} rather than an $n$-dimensional object. In \Cref{fig:benchmark_boundary_dimension_failure}, we show an example where a point cloud ($N = 2,500$) is sampled from two intersecting orthogonal line segments with Gaussian noise $\sigma = .025$. At this noise level, with default neighborhood size $k = 100$, local PCA over-estimates the intrinsic dimension---$99.1\%$ of points are assigned local dimension $2$. HP thus detects the edges of the thickened arms as boundary points, because this is the boundary when the object is treated as $2$-dimensional. Increasing the neighborhood size to $k = 200$ resolves the issue because at the larger neighborhood size, local PCA correctly estimates dimension $1$ at most points except near the intersection. Interestingly, LoDD and BEPS suffer from the same limitation, but increasing the neighborhood size does not help them. DU is able to correctly identify true boundary points at both neighborhood sizes.

\begin{figure}
    \centering
    \vspace{-3mm}
    \includegraphics[width = .59\textwidth]{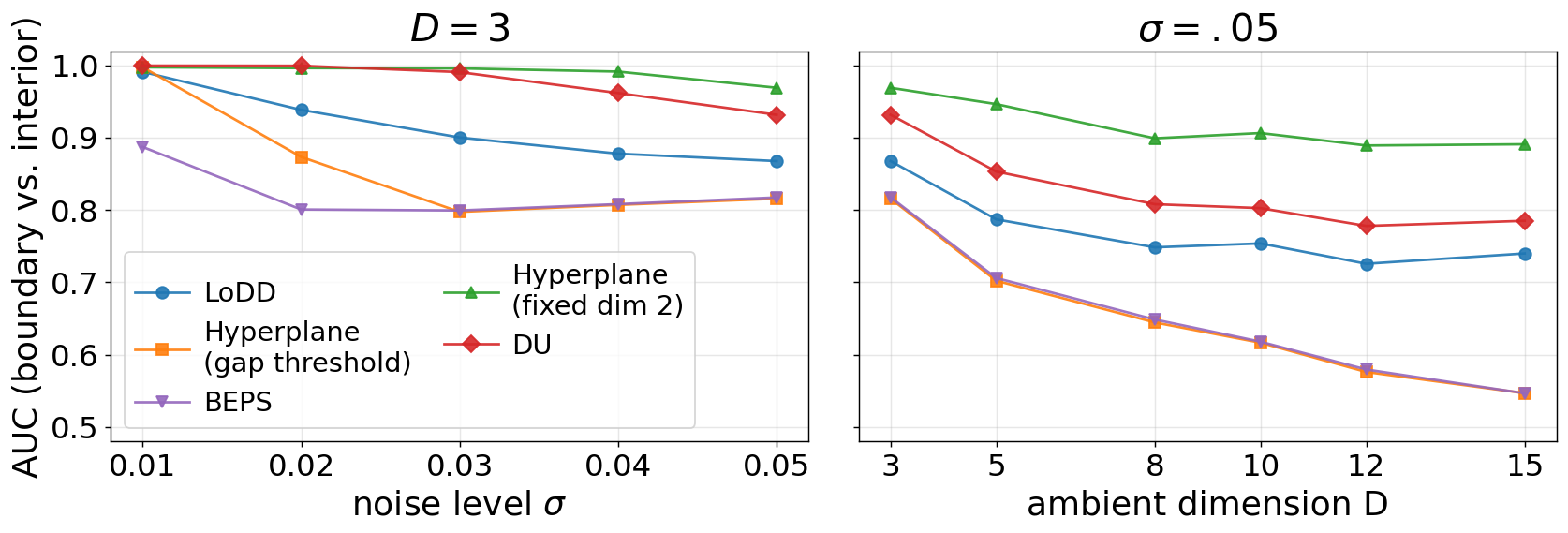}
    \caption{Boundary scores at high noise, for 2D disks ($N = 10,000$) embedded in $\R^D$ with Gaussian noise $\sigma$. We use neighborhood size $k = 100$. Boundary and interior points are defined as points with norm $> .95$ and $< .8$, respectively, with norm computed in the $xy$ plane.}
    \label{fig:disks_increasing_sigma}
\end{figure}

In \Cref{fig:disks_increasing_sigma}, we study each method's boundary-detection performance as we increase the noise level $\sigma$. The point clouds are sampled from 2D disks embedded in $\mathbb{R}^3$ with isotropic Gaussian noise. Again, we see that at higher noise levels, HP often misclassifies interior points as boundary points because local PCA overestimates the dimension as $3$ instead of $2$. An interior point with high $|z|$ coordinate is seen as the ``boundary'' of the thickened disk. However, if we fix dimension $2$ for the local PCA dimension, then HP is able to correctly recover the disk boundary. We plot the AUC values as a function of $\sigma$ in \Cref{fig:disks_increasing_sigma}(a). At the highest noise level ($\sigma = .05$), HP with fixed dimension (AUC $.969$) performs best, followed by DU (AUC $.932$), then LoDD (AUC $.868$), and then finally BEPS (AUC $.818$) and HP with the spectral-gap threshold condition (AUC $.816$). In the high-noise experiments, BEPS and HP (spectral gap threshold) are almost identical because both methods optimize for a hyperplane (in different ways) that maximizes the number of points on a single side of the hyperplane. When the local dimension is overestimated, HP suffers from the same limitations as BEPS. We see the same behavior when the 2D disk is embedded in higher-dimensional $\R^D$ with noise level $\sigma = .05$, in \Cref{fig:disks_increasing_sigma}(b). See \Cref{fig:disks_D3_high_noise} for scatter plots of the boundary scores at $D = 3$ and $\sigma = .05$. This suggests that HP should be used if the intrinsic dimension can be reliably estimated, but otherwise, it is best to use DU.

\paragraph{High intrinsic dimension.}

All methods (hyperplane, DU, LoDD, and BEPS) are able to detect the boundary of an $n$-dimensional object. We sample $N = 10,000$ points from $n$-balls embedded in $\R^n$ for $n = 3, 5, 7, 10$. For each method, the AUC value is $1$ for each $n$, where boundary and interior points are points with norm $> .95$ and $< .8$, respectively. See \Cref{fig:high_n_balls} for boundary scores.

\paragraph{Boundary normals.}

\begin{figure}
    \centering
    \includegraphics[width = .5\textwidth]{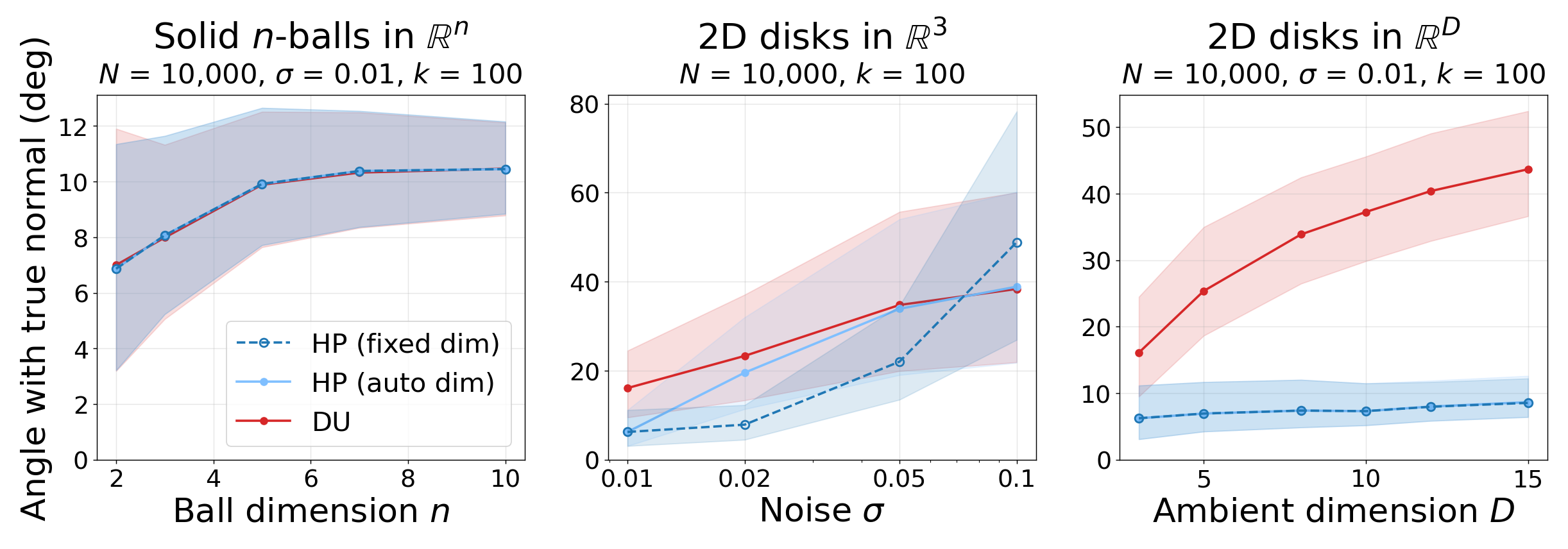}
    \caption{For points $x$ near the boundary of a 2D disk or $n$-ball ($\norm{x} > .95$), we compare the estimated boundary normal with the ground-truth normal $-x/\norm{x}$. We plot median error and IQR.}
    \label{fig:normal_accuracy}
\end{figure}

Both of our boundary detection methods yield normal directions at each point, in addition to numerical scores. For HP, we take the vector $\vec{w}$ (line 5 of \Cref{alg:boundary_hyperplane}) and lift it from its PCA coordinates back into the ambient space to get a normal $\vec{w}_{\text{amb}}$. For DU, the vector $\vec{v_i} := \fracc{k}\sum_j(x_{ij} - x_i)/\norm{x_{ij} - x_i}$ is the same as the vector used by BEPS \citep{beps} and CO \citep{gravitational}. In \Cref{fig:normal_accuracy}, we test the accuracy of these vectors on 2D disks and $n$-balls. For 2D disks in $\R^3$, HP generally yields a more accurate normal direction estimate than DU, with the performance gap narrowing as the noise $\sigma$ increases. For HP, fixing the dimension $2$ works better than using a spectral gap threshold condition at lower noise levels ($\sigma \leq .05$), but at the highest noise level ($\sigma = .1$), the top two local principal components do not align as well with the tangent space. For 2D disks in $\mathbb{R}^D$, the hyperplane normal $\vec{w}_{\text{amb}}$ is significantly more accurate than $\vec{v_i}$ as $D$ increases. For $n$-balls, both methods (and variants) are equally accurate; we suspect that this is usually true for manifolds with codimension zero. Overall, HP generally yields a much more accurate boundary normal direction compared to DU.

\section{Boundary Detection in Spatial Transcriptomics Data}\label{sec:st}

\begin{figure}
    \centering
    \includegraphics[width = .5\textwidth]{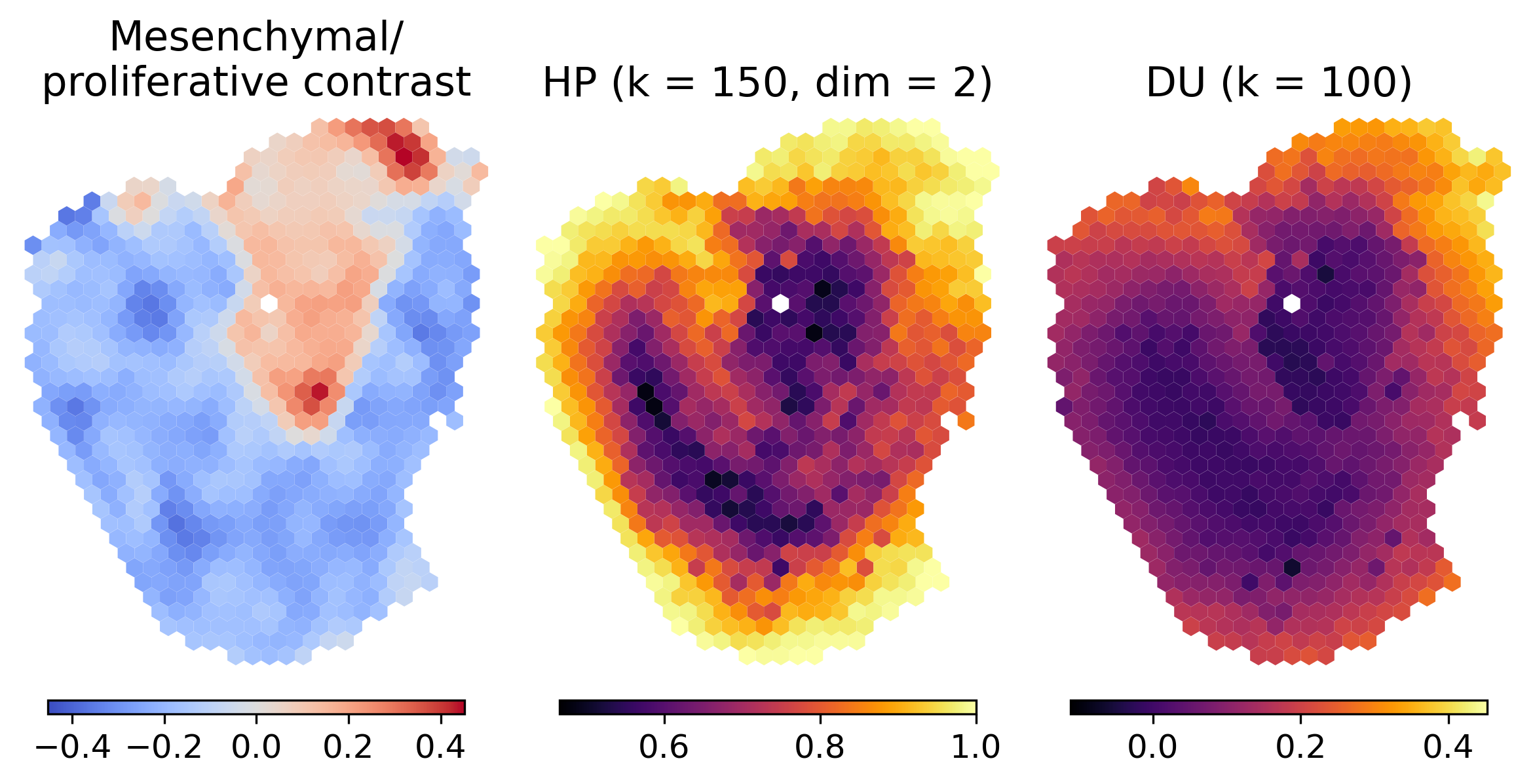}
    \caption{Glioblastoma slice ST1A1. 
    }
    \label{fig:st1a1}
\end{figure} 

As a first application, we apply our methods to spatial transcriptomics data, which measures gene expression across a tissue section, to recover cell-state boundaries in glioblastoma. Glioblastoma is an aggressive brain tumor that diffusely infiltrates the brain, making it difficult to distinguish regions where tumor cells are abundant or sparse. Malignant cells also occupy distinct transcriptional states \citep{neftel2019states} that may be amenable to different therapies, further motivating the identification of boundaries within the tumor. We analyze tissue slice ST1A1 from \citet{aldalahmah2023tissue}.

Each capture area measures the combined expression of multiple cells. To interpret our boundary scores, we use published cell-state abundance estimates obtained by Robust Cell Type Decomposition (RCTD) \citep{cable2022rctd}, which deconvolves this expression into contributions from different cell states. Following \citet{aldalahmah2023tissue}, we combine the malignant-cell subpopulations gl\_Mes1 and gl\_Mes2 into an astrocyte-like/mesenchymal group (Mes), and gl\_Pro1 and gl\_Pro2 into a proliferative group (Pro). These represent distinct malignant states that have prognostic and therapeutic implications. \Cref{fig:st1a1} shows their spatial distributions and the mesenchymal/proliferative contrast, calculated as Mes minus Pro.

In \Cref{fig:st1a1}, both HP and DU assign higher boundary scores along the tissue edge and near parts of the transition between Mes- and Pro-enriched regions. 
After pre-processing (see \Cref{sec:ST_details}), we reduce gene expression to $30$ principal components
and concatenate these with the two spatial coordinates, yielding a point cloud of size $N = 741$ (the number of spatial spots) in $\mathbb{R}^{32}$. We view this as the graph of an expression function over spatial coordinates, motivating fixed dimension two for HP. For both methods, we construct each spot's neighborhood by
taking the intersection of nearest neighbors with respect to both spatial coordinates and gene expression coordinates.
Our results suggest that our methods can detect internal transitions between tumor states without using cell-state labels or predefined tumor-marker genes.

\section{LLM Activation Geometry}\label{sec:llm}

As as second application, we reproduce the character-counting manifold of \citet{gurnee2026} on two open-weight models (Pythia-2.8B and Llama-3.1-8B) and apply our boundary-detection methods. Using text from the C4 (English) data set \citep{c4}, we strip newlines and then re-insert them every $k$ characters, rounded to the nearest word boundary, for $k \in \{15, 20, \ldots, 150\}$. This produces $28$ versions of each document, each with fixed line width $w$. Every non-newline token has a character position $i$ within its line of width $w$. We extract the layer-$4$ residual-stream activations and average the activations for each $(i, w)$, yielding point clouds of size $N \approx 2,300$ for each model. \Cref{fig:pythia} shows the Pythia-2.8B point cloud; see \Cref{fig:llama} for Llama-3.1-8B. See \Cref{sec:llm_details} for details of our procedure.

In \Cref{fig:pythia}(b), both of our boundary detectors correctly assign higher scores to the ground-truth boundary of the averaged manifold, which is $\{(i, w) \mid i = 0 \text{ or } i = w \text{ or } w = 15 \text{ or } w = 150\}$. For HP, we use dimension $n = 2$ because the surface is intrinsically parameterized by character position $i$ and line width $w$. We also show the vector fields of estimated boundary normals in \Cref{fig:pythia}(c) for the top 10\% scoring points; the vectors tend to point inwards on the surface, as they should. 

Lastly, we test HP and DU on the un-averaged, per-token point cloud (size $N \approx 250$k per model). See \Cref{fig:pythia}(d). This is much more challenging because the intrinsic dimension of the space is much higher (and unknown), and there are other, unknown boundaries in the token point cloud. After calculating boundary scores at each point, we average the boundary scores for the tokens at each $(i, w)$. The averaged boundary scores are consistently higher at the ground-truth boundary.

\begin{figure}
    \centering
    \subfloat[]{\includegraphics[width = .25\textwidth]{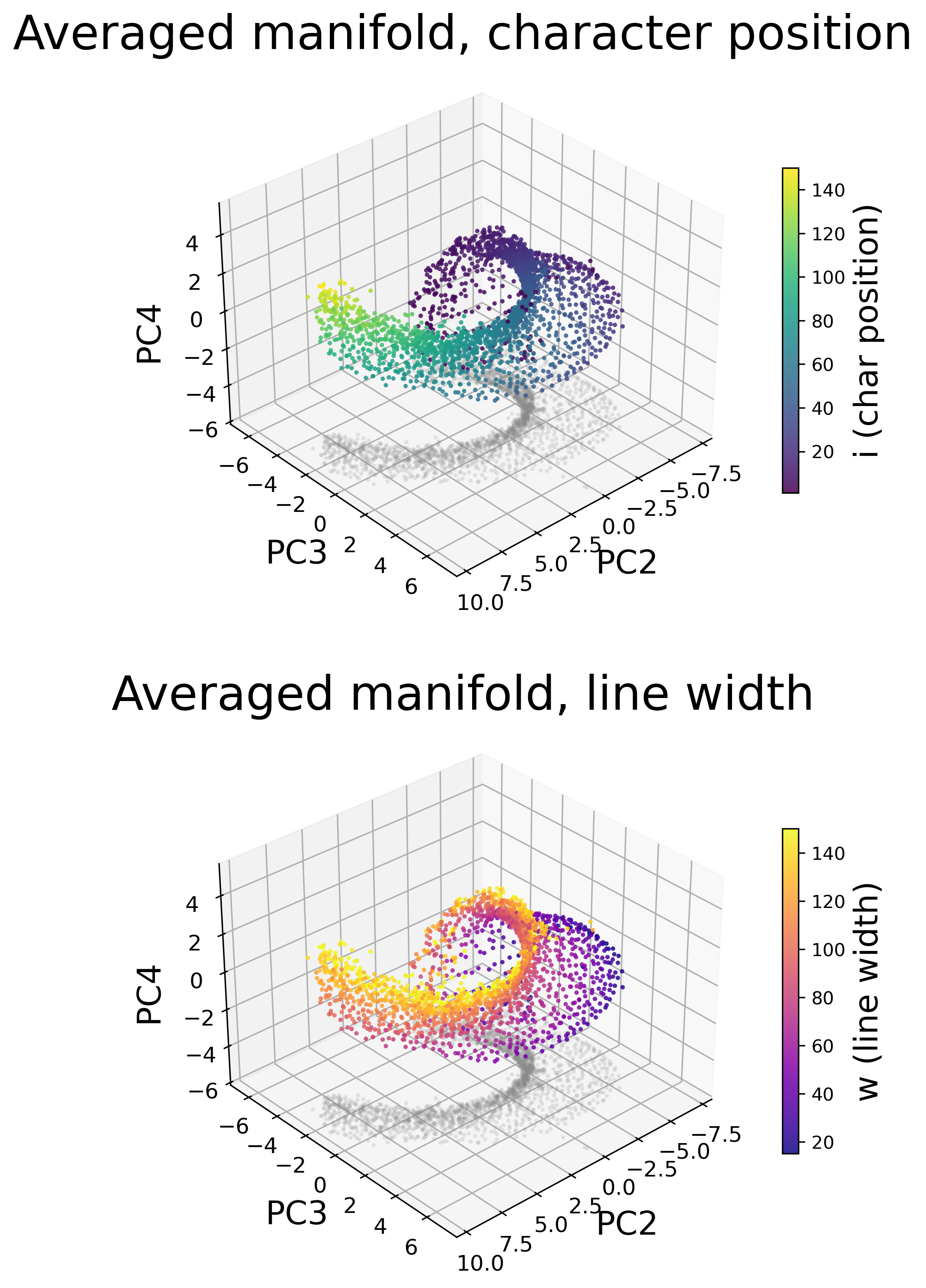}}
    \subfloat[]{\includegraphics[width = .23\textwidth]{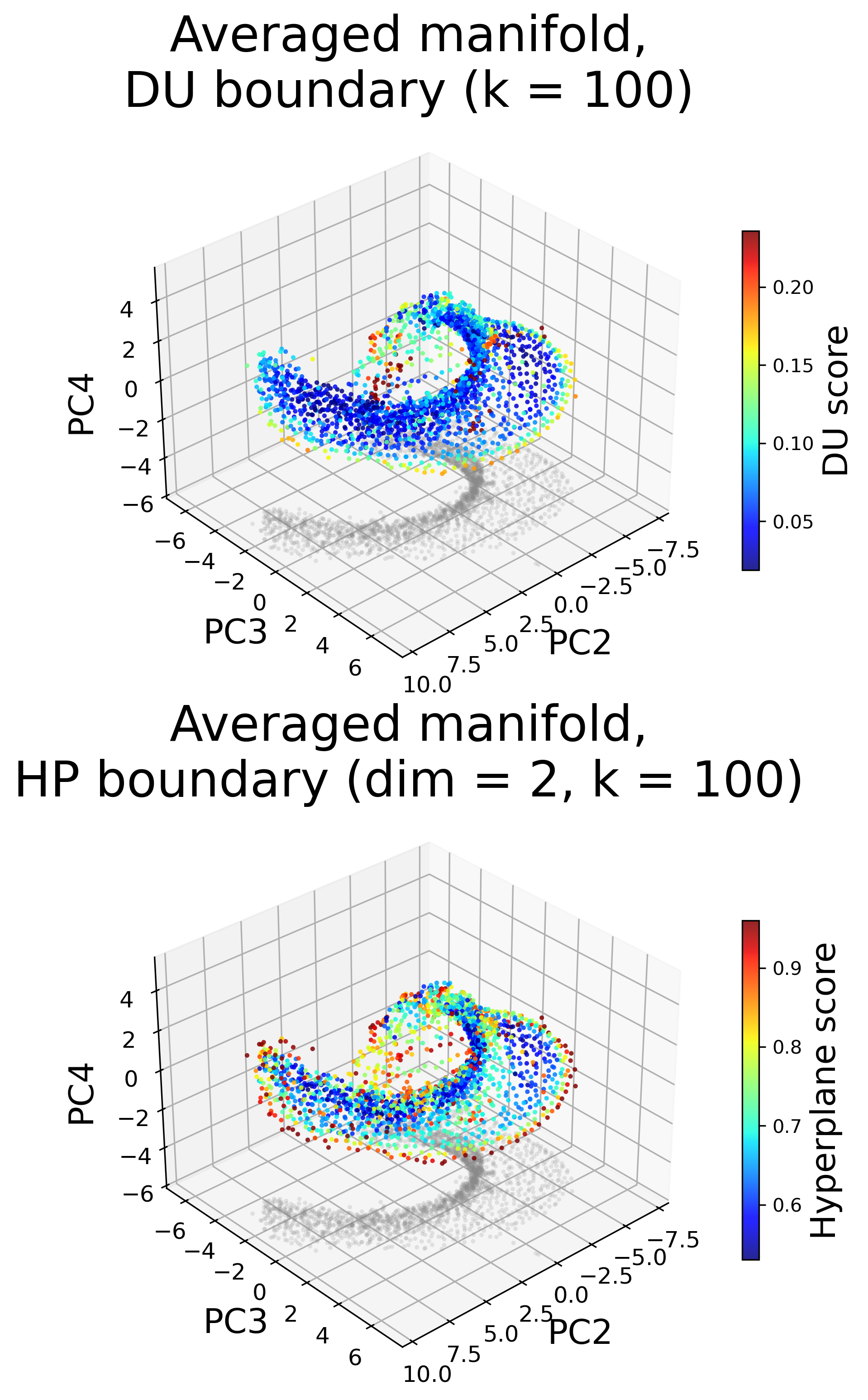}}
    \subfloat[]{\includegraphics[width = .27\textwidth]{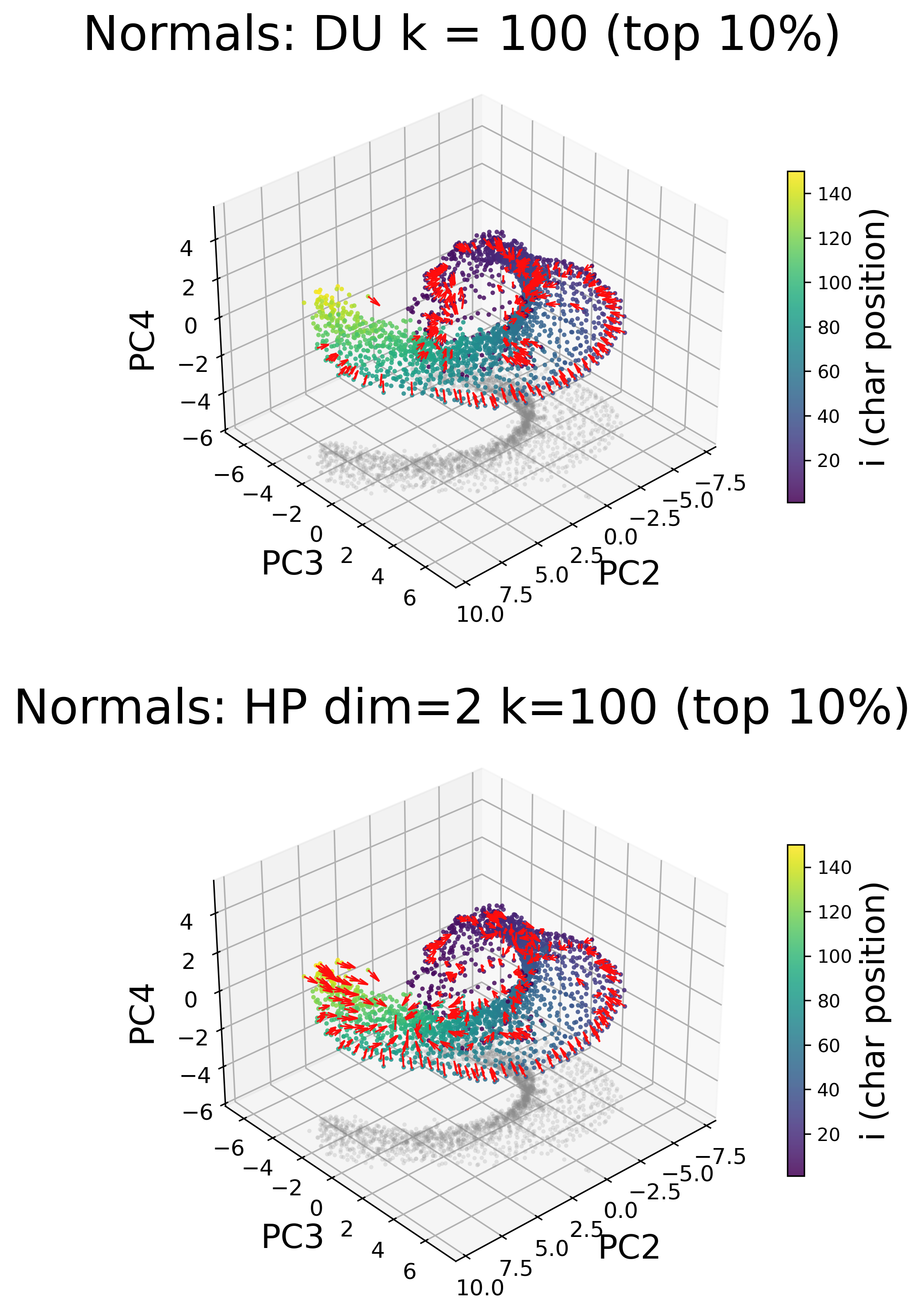}}
    \subfloat[]{\includegraphics[width = .22\textwidth]{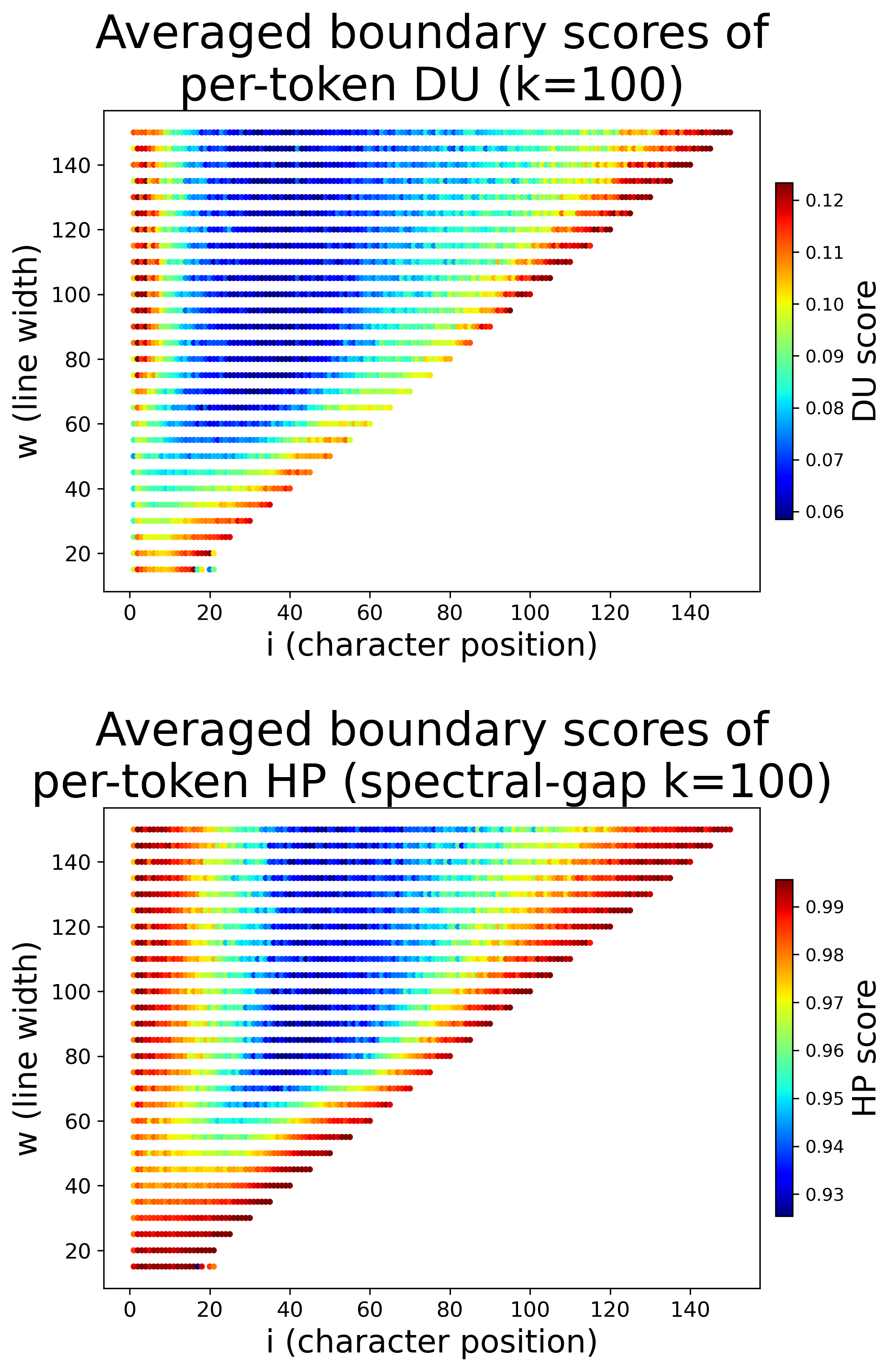}}
    \caption{Boundary detection on the character-counting manifold for Pythia-2.8B.}
    \label{fig:pythia}
\end{figure}

\section{Conclusions}

We presented two new methods for identifying boundary points and calculating boundary normal vectors in high-dimensional, noisy data. We justified the hyperplane method mathematically, and we rigorously tested both methods empirically against others in the literature on synthetic data. Finally, we showed that our boundary detection methods are effective on real-world applications by using them to identify a tumor boundary in a spatial transcriptomics data set and to characterize the boundary of a manifold in LLM activations.

\FloatBarrier
\subsection*{AI use statement}

We used generative AI to write almost all of the code and to assist with isolated steps within the proofs of \Cref{lem:prob_full}, \Cref{lem:third_moment_bound}, \Cref{lem:f_tau_approximation}, and \Cref{lem:cov_full}.

We have reviewed all AI-assisted work. In particular, we have reviewed the code and checked all mathematical arguments for correctness. We take responsibility for the final content of this work, including text, claims or artifacts produced with the aid of generative AI.

\subsubsection*{Acknowledgments}
Hickok was supported by NSF grant DMS-2303402. Al Dalahmah was partially supported by NIH grant 1R21NS136647 NINDS.  Blumberg was partially supported by NSF grants DMS-2311338 and IIS-2552120.

\bibliography{references}
\bibliographystyle{iclr2027_conference}

\appendix

\section{Additional Experimental Figures}

\subsection{Hyperplane optimization in half and full balls}

To illustrate the intuition for HP, we calculate the vector $\vec{w}$, the solution to \Cref{eq:boundary_optimization}, for points sampled from 2D and 3D balls and half balls. In \Cref{fig:boundary_ball_model}(a), we showed $\vec{w}$ and its orthogonal hyperplane (line) for $k = 10$ points sampled from a 2D half ball. In \Cref{fig:more_balls_hyperplane_opt}, we show $\vec{w}$ and its orthogonal hyperplane for $k = 50$ points in 2D half and full balls and for $k = 100$ points in 3D half and full balls. In all cases, the ``accuracy'' (the percentage of points $\z$ such that $\vec{w} \cdot \z \geq 0$) is close to $1$ for the half balls and close to $.5$ for the full balls.

\begin{figure}
\subfloat[]{\includegraphics[width = .25\textwidth]{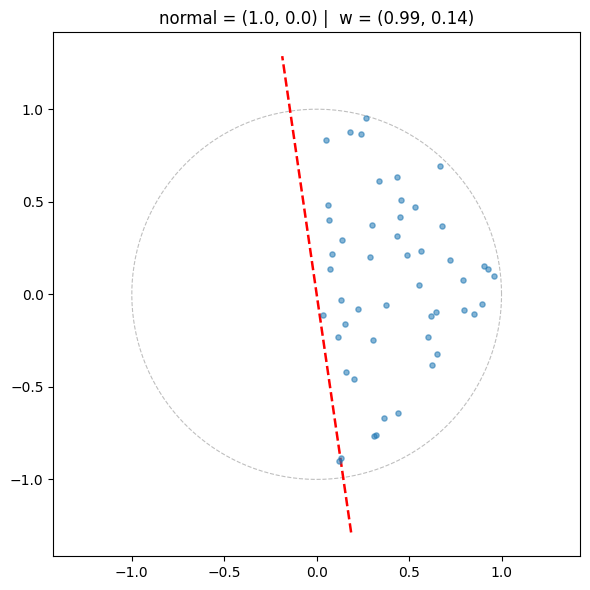}}
    \subfloat[]{\includegraphics[width = .25\textwidth]{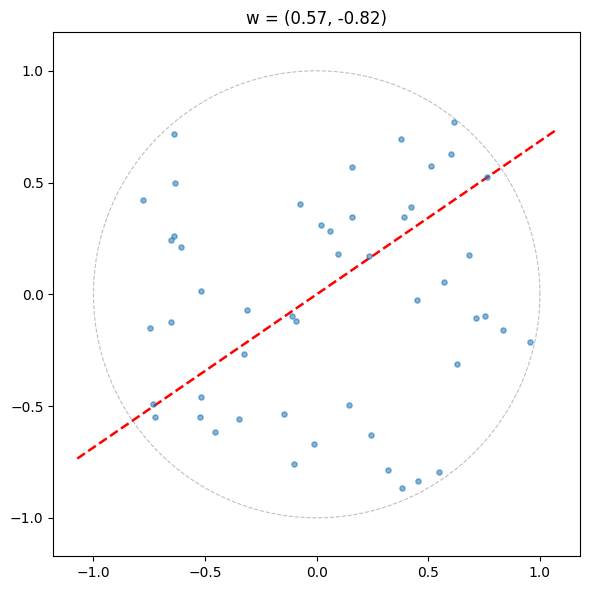}}
    \subfloat[]{\includegraphics[width = .25\textwidth]{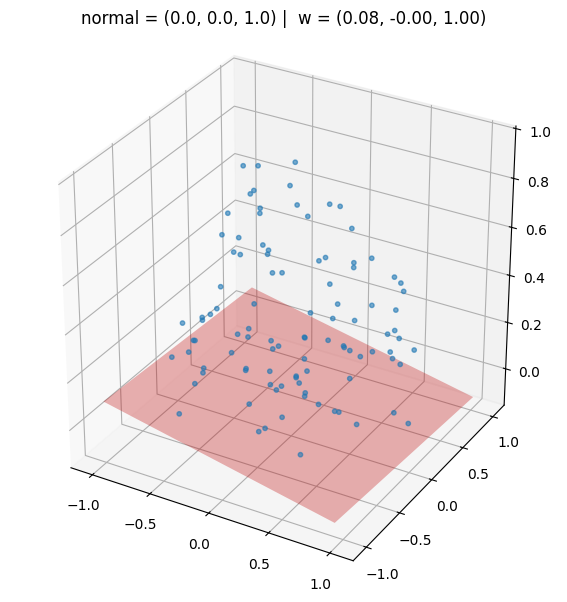}}
    \subfloat[]{\includegraphics[width = .25\textwidth]{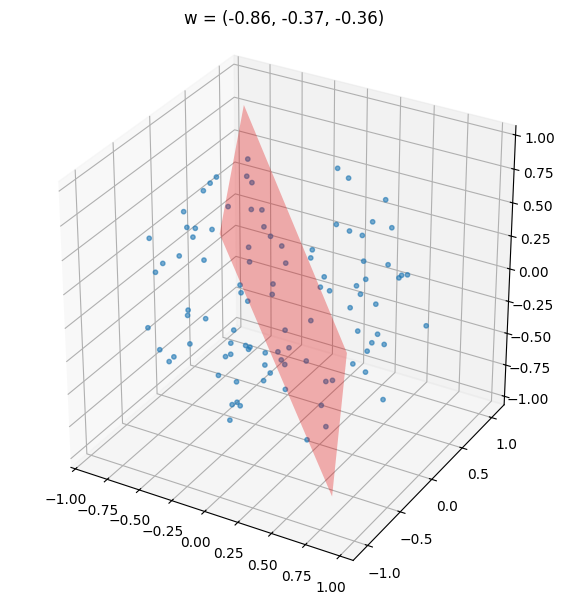}}
    \caption{We calculate the vector $\vec{w}$ that is the solution to \Cref{eq:boundary_optimization} in our hyperplane method, where the accuracy of $\vec{w}$ is the percentage of points $\{z_i\}$ such that $\vec{w} \cdot \z_i \geq 0$. (A) 2D half ball with normal $(1,0)$ and $k=50$ points, with accuracy $.98$,
(B) 2D full ball with $k=50$ and accuracy $.5$,
(C) 3D half ball with normal $(0,0,1)$ and $k=100$, with accuracy $.99$, and
(D) 3D full ball with $k=100$ and accuracy $.6$.}
    \label{fig:more_balls_hyperplane_opt}
\end{figure}

\subsection{Non-uniform density}

In \Cref{sec:density}, we tested boundary detection methods on a 2D disk with non-uniform density and reported AUC values in \Cref{tab:density_well_auc}. In \Cref{fig:density}, we show scatter plots of the pointwise boundary scores for each method.
\begin{figure}
    \centering
    \includegraphics[width = \textwidth]{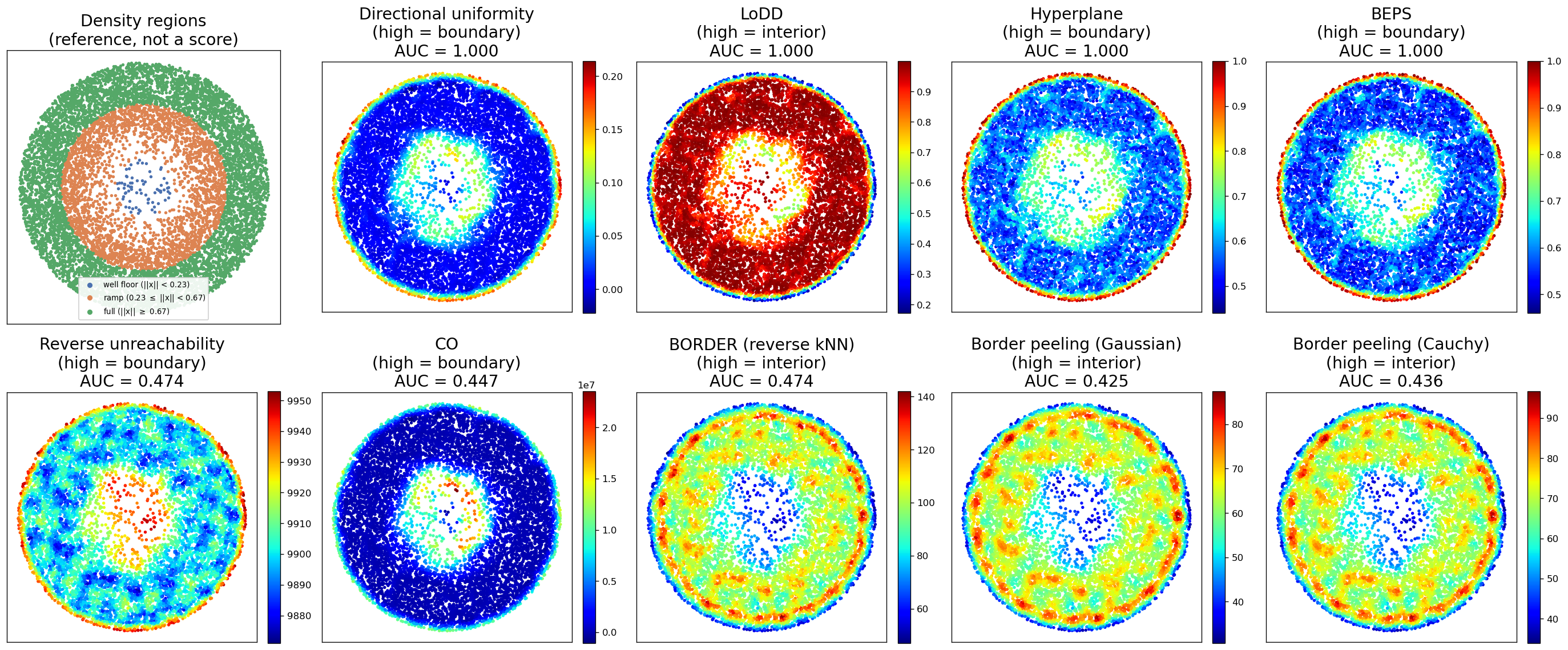}
    \caption{We sample $N = 10,000$ points non-uniformly from a 2D disk. In the interior, there is a ``well'' ($r < .23$) where the density is $1/9$ that of the density in the outer region ($r > .67$). The density increases linearly in between. For each method, we calculate the AUC value for boundary points ($\norm{x} > .98$) vs. interior well points ($\norm{x} < .23$), using neighborhood size $k = 100$. In the top row, we show the methods that are able to handle non-uniform density. In the bottom row, we show the methods that confuse boundary points and low-density interior points.} 
    \label{fig:density}
\end{figure}

\subsection{Surfaces with non-trivial geometry}

In \Cref{sec:surfaces}, we calculated boundary scores for the hyperplane method, DU, LoDD, and BEPS on a cone and a saddle-skewer data set with low noise levels. In \Cref{fig:benchmarks_appendix} below, we show the boundary scores for four more synthetic data sets.

\begin{figure}
    \includegraphics[width = \textwidth]{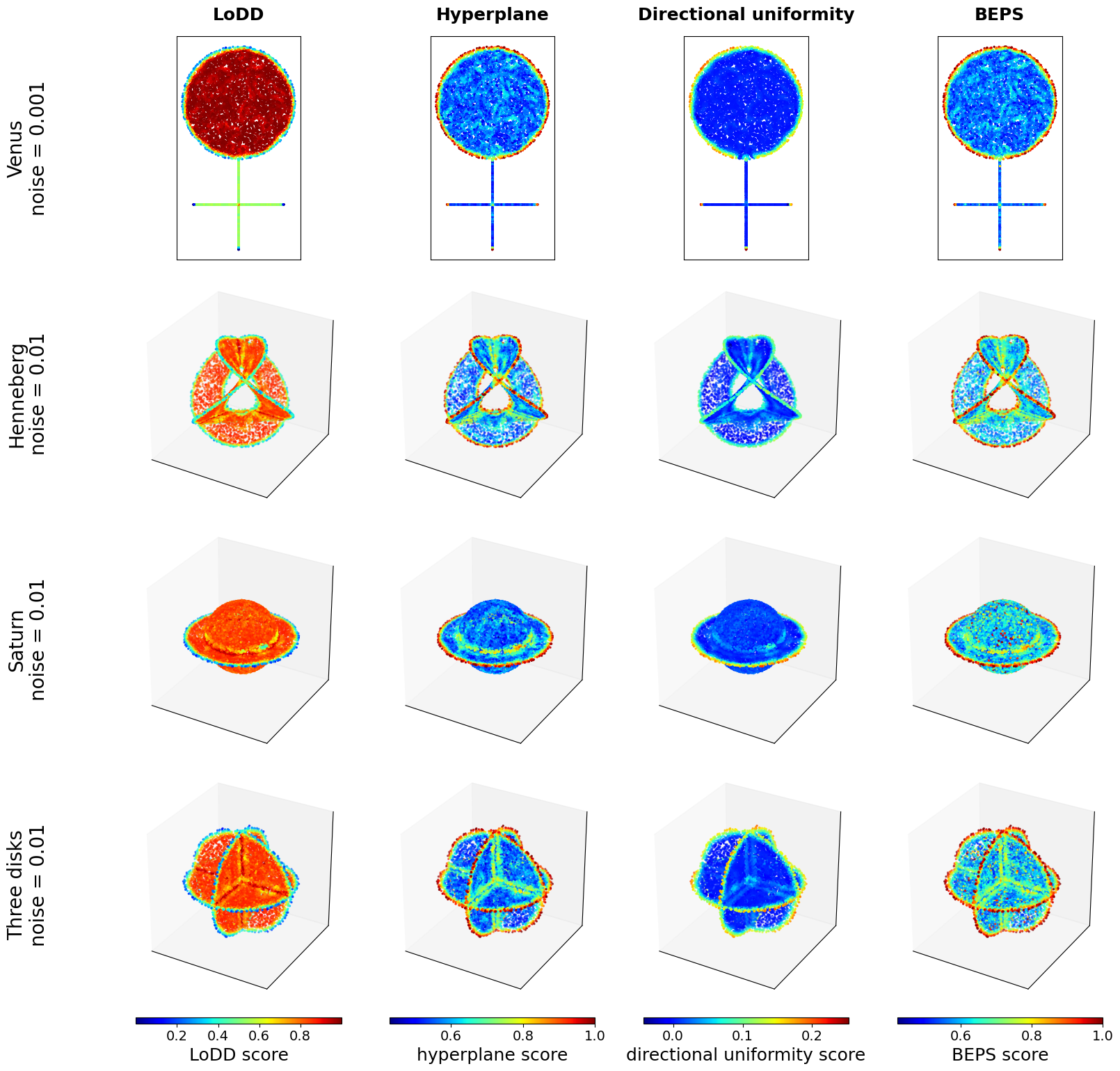}
    \caption{Boundary scores for four synthetic data sets with low noise levels. Each point cloud has $N = 10,000$ points, and we use neighborhood size $k = 100$.}
    \label{fig:benchmarks_appendix}
\end{figure}

\subsection{High ambient dimension}

In \Cref{sec:high_ambient}, we calculated boundary scores for the hyperplane method, DU, LoDD, and BEPS for surfaces in high ambient dimension with Gaussian noise. \Cref{fig:benchmarks_high_ambient} showed the Henneberg minimal surface in $\R^{15}$ and \Cref{fig:disks_high_ambient} analyzed the boundary scores for disks in $\R^D$. In \Cref{fig:benchmarks_high_ambient_appendix}, we show the boundary scores for three more surfaces in $\R^{15}$, and in \Cref{fig:disks_highdim_scatter}, we show the boundary scores for $2$D disks in $\R^{15}$ with noise.

\begin{figure}
    \includegraphics[width = \textwidth]{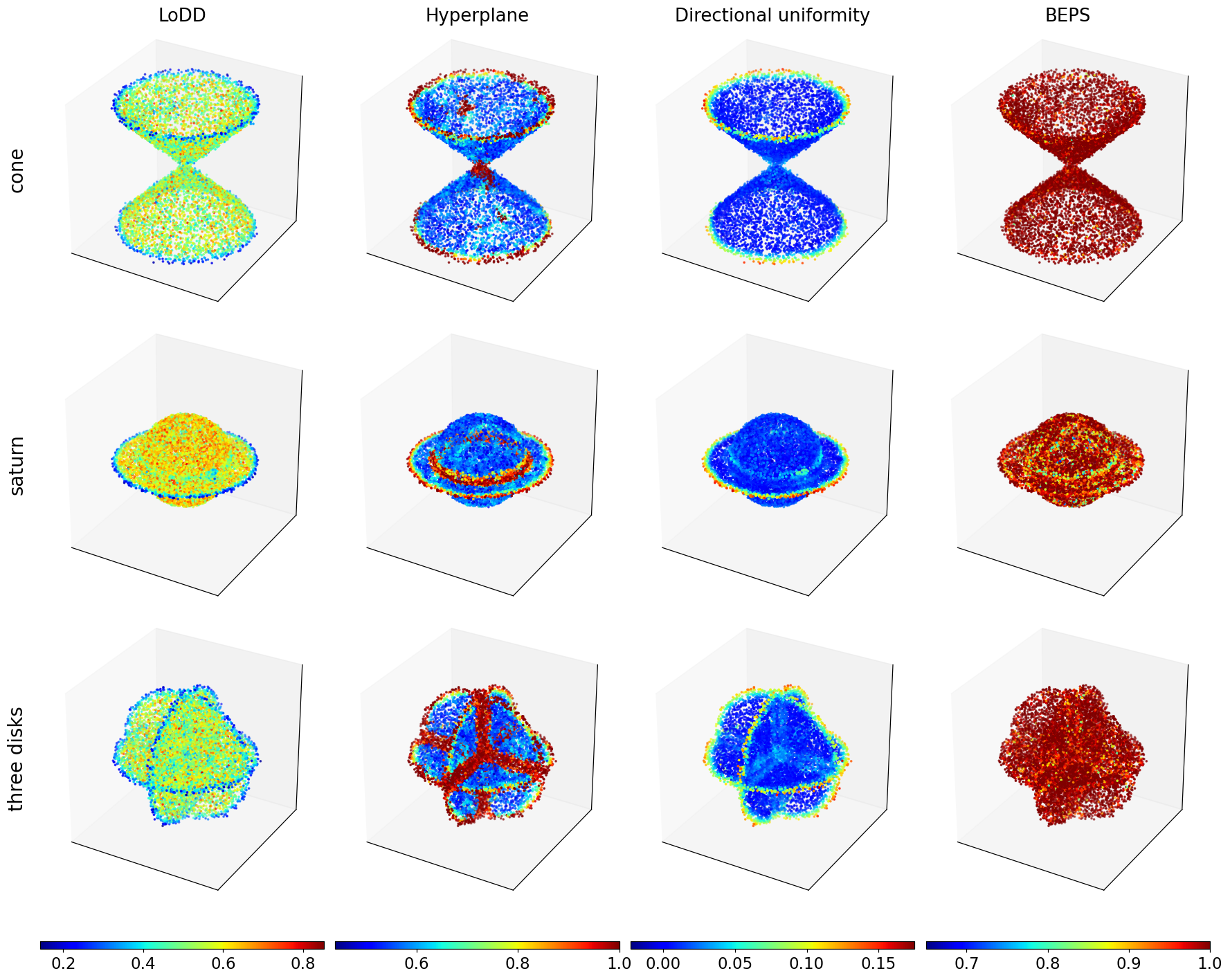}
    \caption{Boundary scores for three synthetic data sets embedded in $\R^{15}$ with noise. Scores here are plotted in $\R^3$.}
    \label{fig:benchmarks_high_ambient_appendix}
\end{figure}

\begin{figure}
\includegraphics[width =\textwidth]{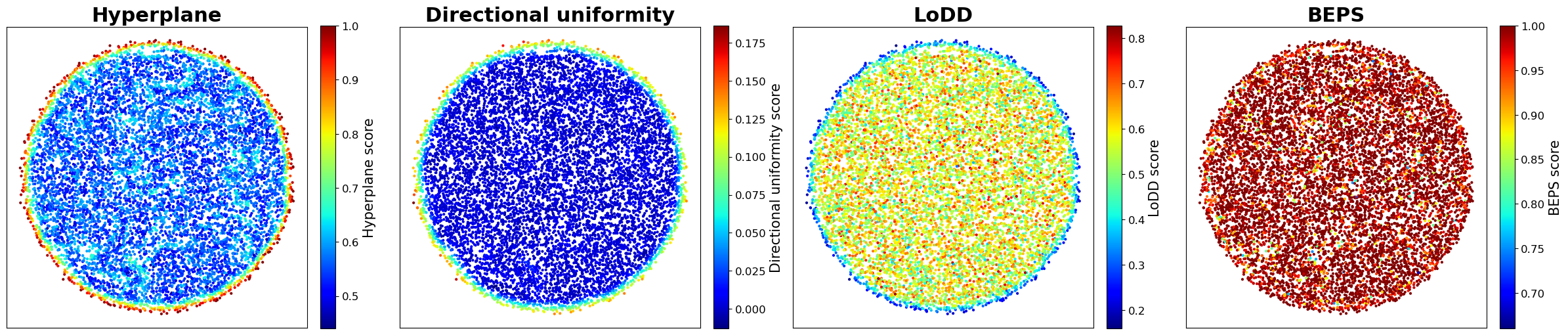}
\caption{Boundary scores on disks embedded in $\mathbb{R}^{15}$ with Gaussian noise $\sigma = .01$, with $N = 10,000$ points. We use neighborhood size $k = 100$. We plot only the first two coordinates.}
\label{fig:disks_highdim_scatter}
\end{figure}
\subsection{High noise levels}

In \Cref{sec:high_noise}, we tested our hyperplane and DU methods against LoDD and BEPS on 2D disks embedded in $\R^D$ with increasing noise level $\sigma$. In \Cref{fig:disks_D3_high_noise}, we show the boundary scores for disks in $R^3$ at the highest noise level $\sigma = .05$.

\begin{figure}[H]
    \centering
   \includegraphics[width = \textwidth]{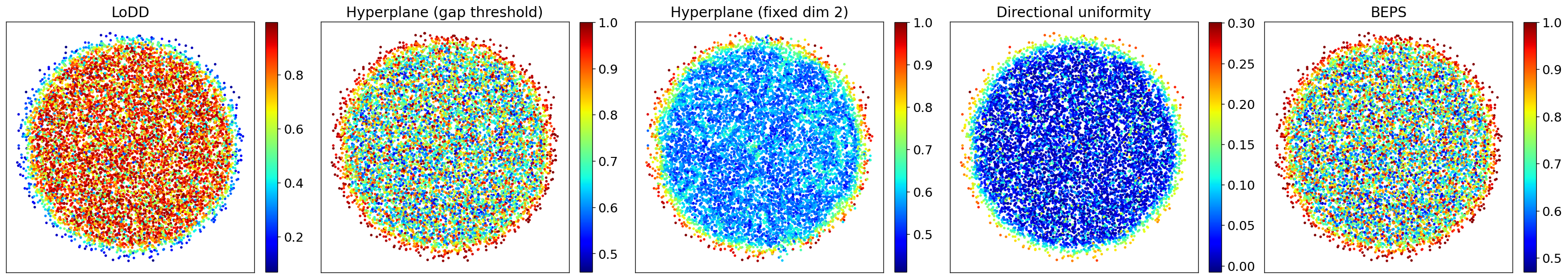}
    \caption{Boundary scores for 2D disks embedded in $\R^3$ with isotropic Gaussian noise $\sigma = .05$, projected here to $\R^2$.}
    \label{fig:disks_D3_high_noise}
\end{figure}

\subsection{High intrinsic dimension}

In \Cref{fig:high_n_balls}, we show the boundary scores for points sampled from a $10$-dimensional ball in $\R^{10}$. All methods achieve AUC $1$ except for reverse unreachability (AUC $.01$). For the AUC calculation, we define boundary and interior points as points with norm $> .95$ and $.8$, respectively.
\begin{figure}[H]
    \centering
    \includegraphics[width = \textwidth]{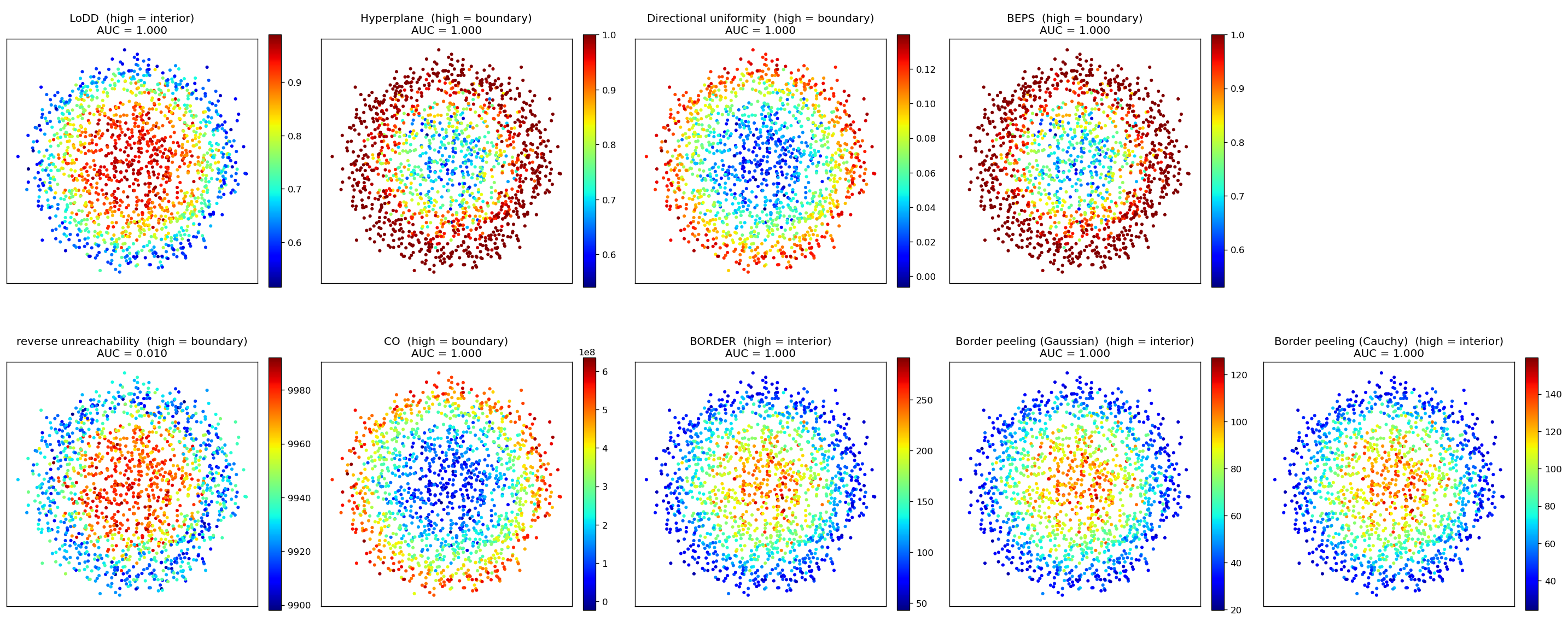}
    \caption{Boundary scores on a point cloud with $N = 10,000$ points sampled from a $10$-dimensional ball in $\R^{10}$, with isotropic Gaussian noise ($\sigma = .01$). The ground-truth boundary is difficult to display in $2$D because it is $9$-dimensional, so instead we take a thin 2D slice (a slightly thickened disk) the ball, plotting the $1500$ points closest to the $xy$ plane. Because of the curvature of the boundary (a sphere), the boundary appears thicker than the 2D balls in earlier figures, but this is an artifact of the way we are taking a slice and projecting to $\R^2$.}
    \label{fig:high_n_balls}
\end{figure}

\subsection{Spatial transcriptomics}

In \Cref{sec:st}, we analyzed glioblastoma slice ST1A1 from \citet{aldalahmah2023tissue}. In \Cref{fig:st1a1}, we showed HP and DU scores at neighborhood sizes $k = 150$ and $k = 100$, respectively. In \Cref{fig:st1a1_appendix} below, we show the boundary scores at neighborhood sizes $k = 75, 100, 150$ for each method.

\begin{figure}
    \centering
    \includegraphics[width = \textwidth]{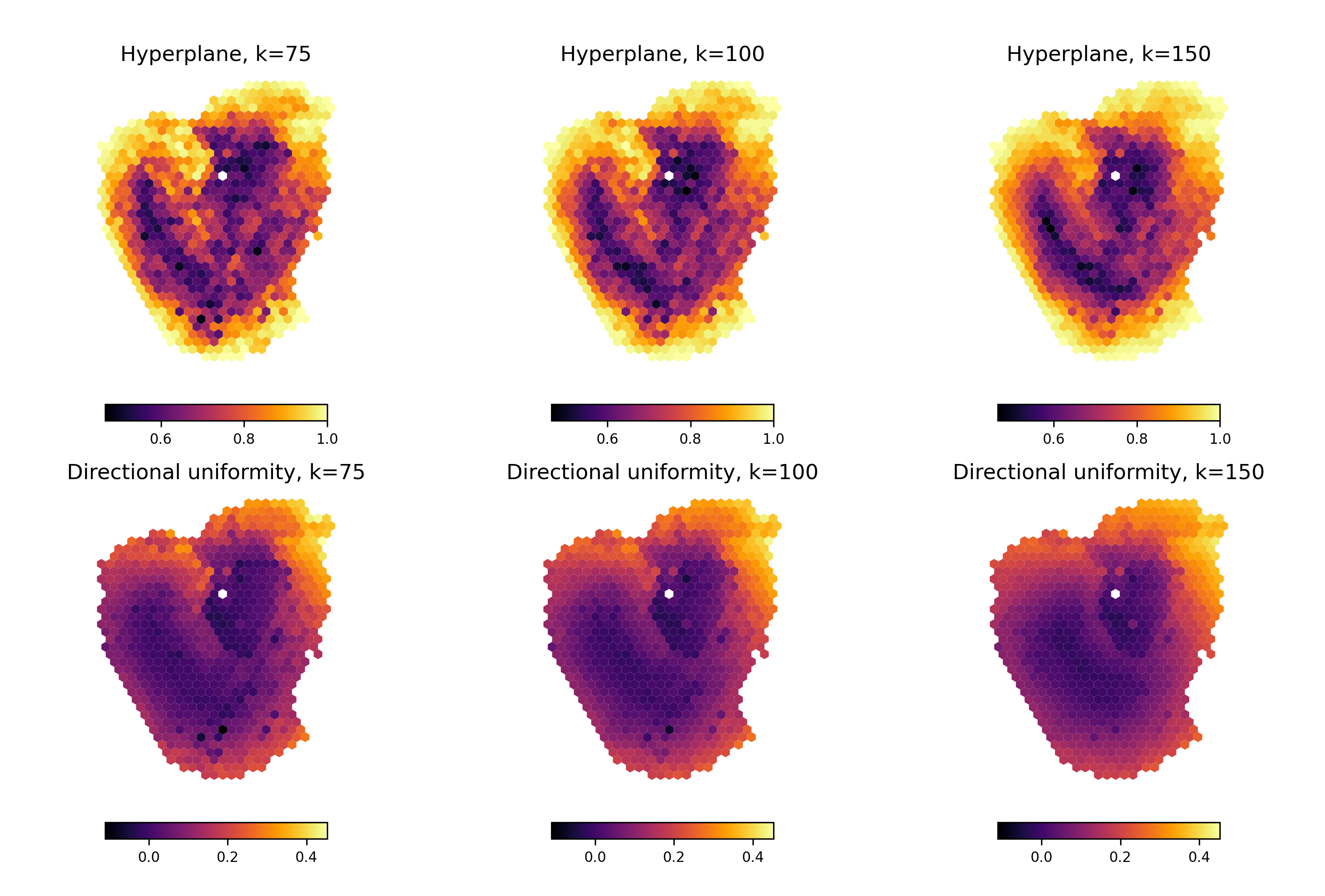}
    \caption{HP and DU boundary scores on the ST1A1 glioblastoma slice at neighborhood sizes $k = 75, 100, 150$.}
    \label{fig:st1a1_appendix}
\end{figure}
\subsection{Llama-3.1-8B activation geometry}

We reproduce the character-counting manifold of \citep{gurnee2026} for Llama-3.1-8B using the methodology in \Cref{sec:llm_details}. In \Cref{fig:llama}, we show the averaged manifold and our detected boundaries. We also show the average boundary scores at each character position $i$ and line width $w$ when we compute boundary scores for the un-averaged, per-token point cloud. See \Cref{fig:pythia} for Pythia-2.8B.

\begin{figure}[H]
    \centering
    \subfloat[]{\includegraphics[width = .25\textwidth]{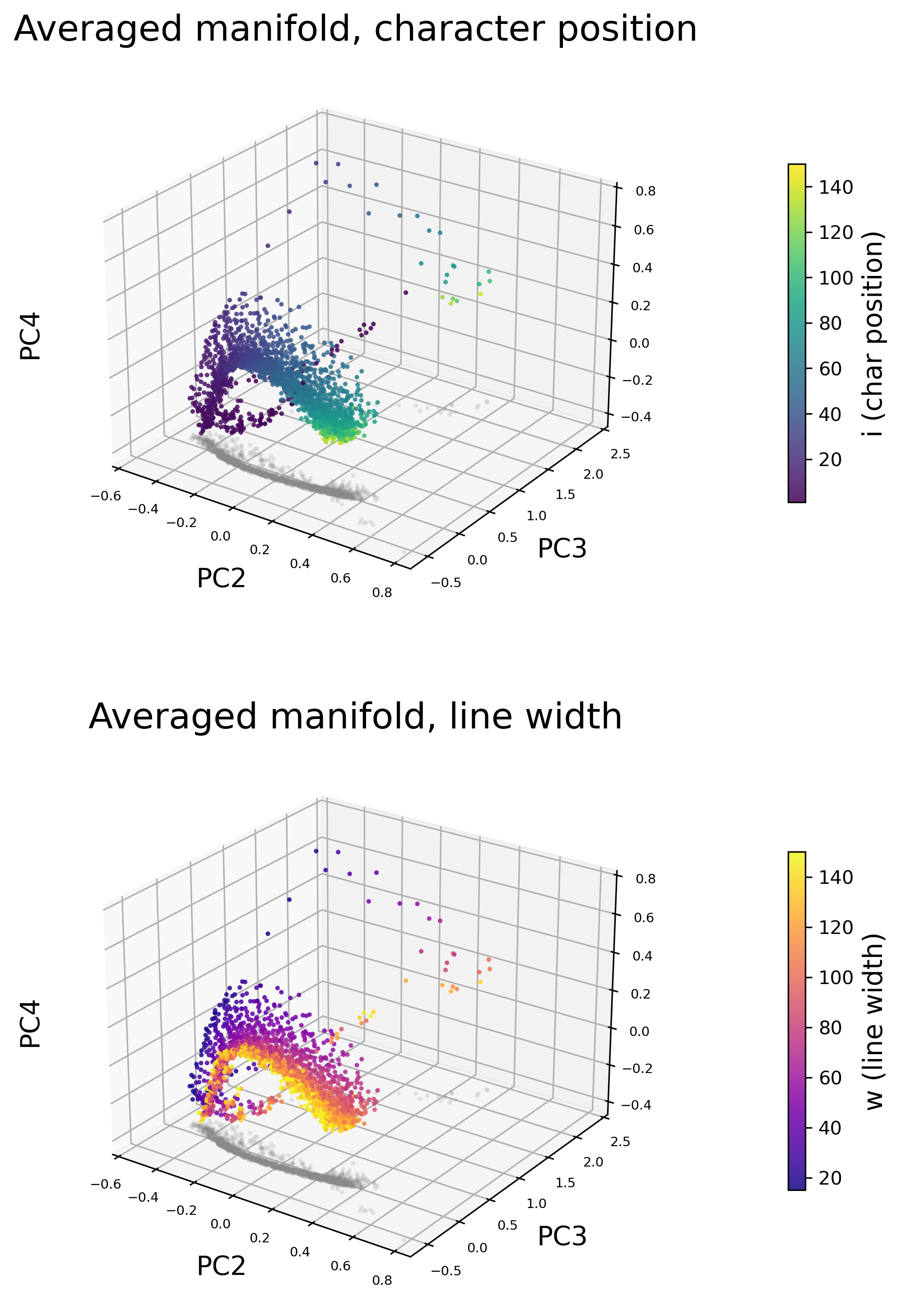}}
    \subfloat[]{\includegraphics[width = .25\textwidth]{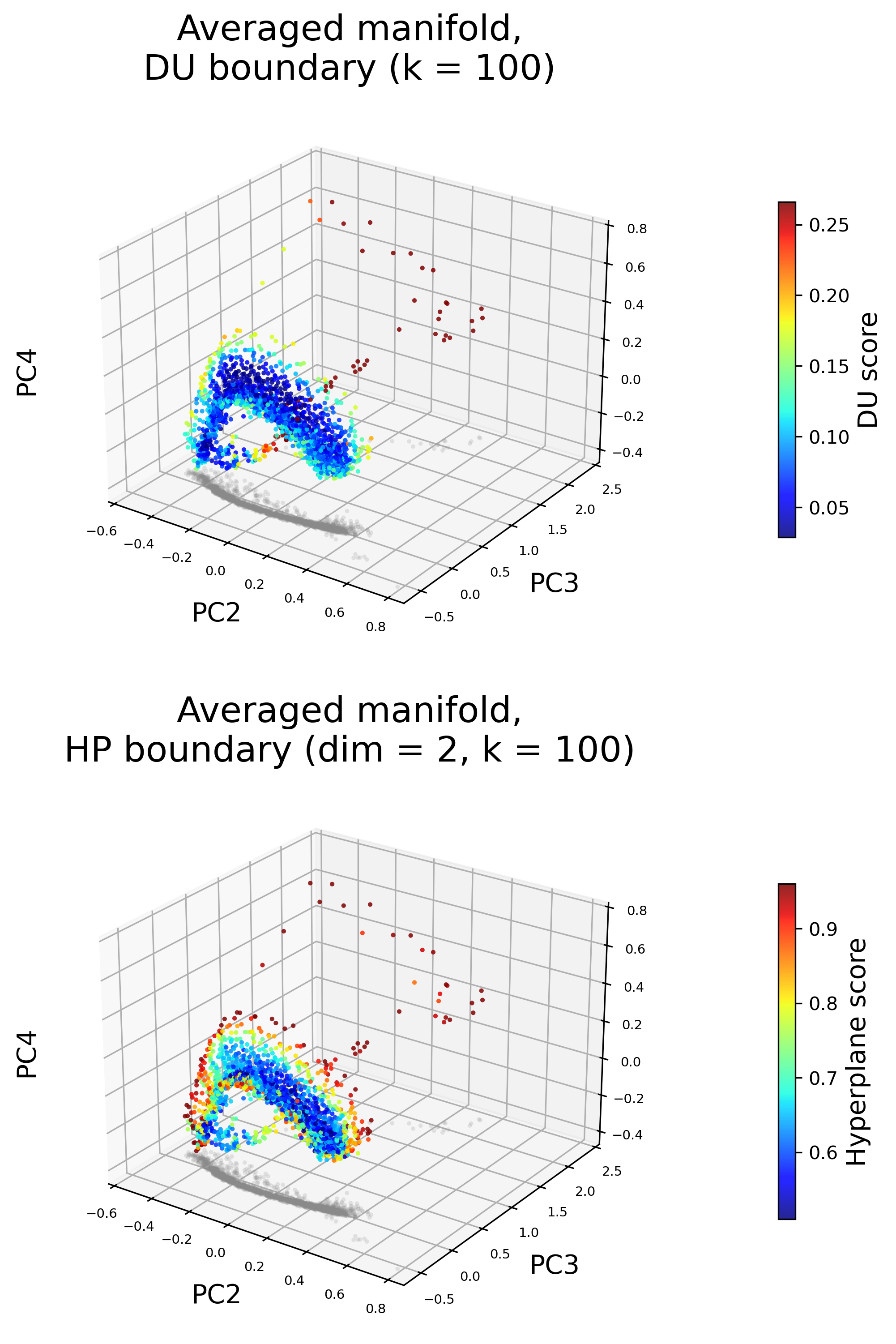}}
    \subfloat[]{\includegraphics[width = .25\textwidth]{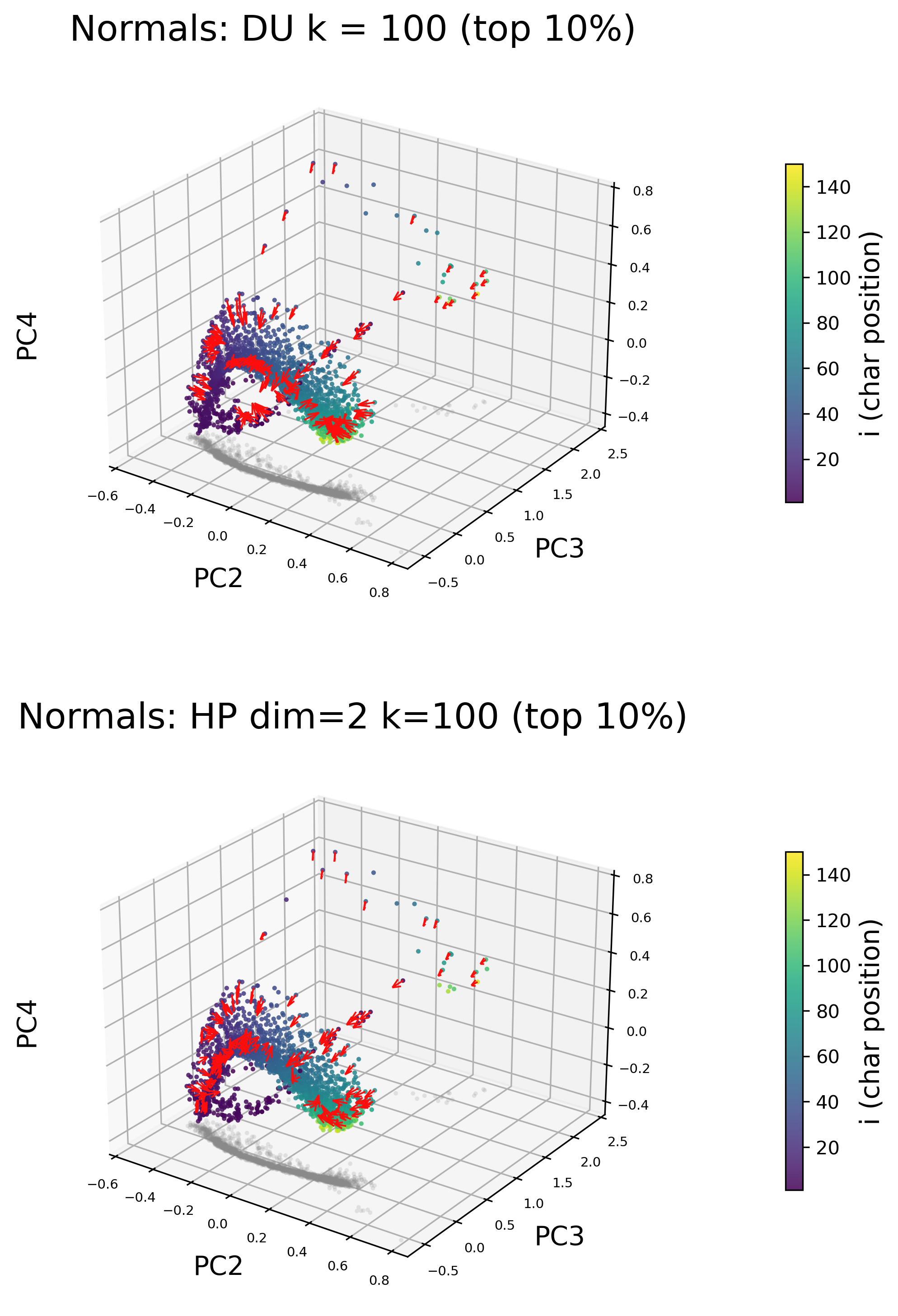}}
    \subfloat[]{\includegraphics[width = .25\textwidth]{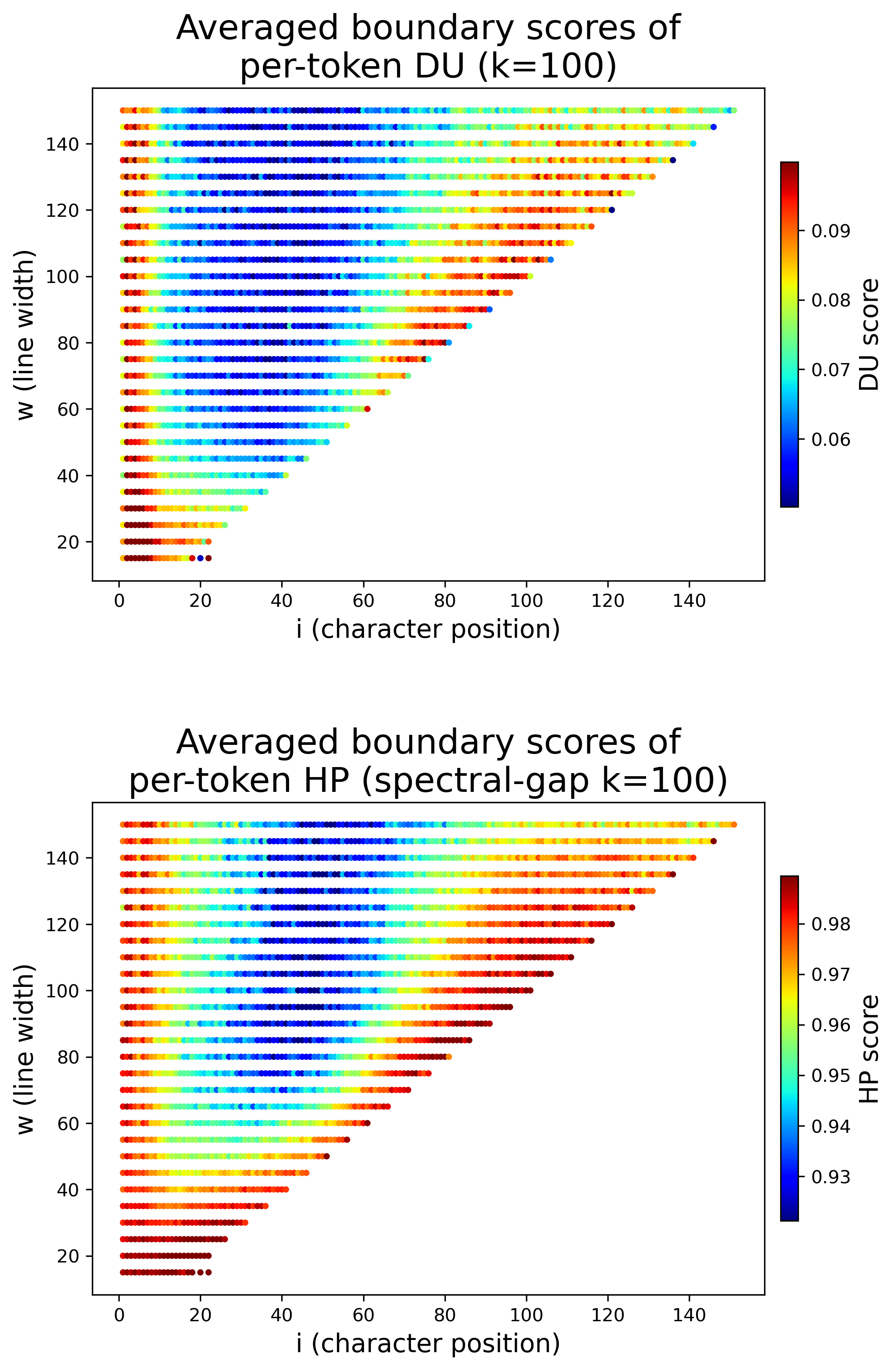}}
    \caption{Boundary detection on the character-counting manifold for Llama-3.1-8B. In (c), we calculate normal vectors only for the top 10\% scoring boundary points.}
    \label{fig:llama}
\end{figure}

\FloatBarrier 

\section{Mathematical proofs}

\subsection{Proof of \Cref{thm:full} for interior manifold points}\label{sec:math_full}

Suppose that $\{\vec{z_i}\}_{i=1}^k$ are sampled independently and uniformly at random from an $n$-dimensional ball. Let $\vec{w} = \frac{\sum_{i} \vec{z_i}}{\norm{\sum_i \vec{z_i}}}$ and let $S_{nk}$ be defined as in \Cref{eq:snk}. Define
\begin{align*}
    X_i &:= \mathbf{1}_{\vec{w} \cdot \vec{z_i} \geq 0}\,, \\
    Q_{nk} &:= \p[ \vec{w} \cdot \vec{z_1} \geq 0]\,.
\end{align*}
\begin{proof}[Proof of \Cref{thm:full}]
     By Chebyschev's inequality,
    \begin{equation*}
        \p\Big[ \big\vert S_{nk}/k - \fracc{2} \big\vert > \epsilon \Big] \leq \frac{\var(S_{nk})}{k^2\epsilon^2} = \frac{Q_{nk}(1 - Q_{nk})}{k\epsilon^2} + \frac{\cov(X_1, X_2)}{\epsilon^2}\,.
    \end{equation*}
    The first summand is bounded above by $k^{-1}\epsilon^{-2}$ because $0 \leq Q_{nk} \leq 1$. By \Cref{lem:cov_full}, there are constants $C_1, C_2$ such that $|\cov(X_1, X_2)| \leq C_1 \sqrt{\frac{n}{k}} + C_2 \frac{n}{k}$ for all $k \geq 3$. Therefore,
    \begin{equation*}
        \p\Big[ \big\vert S_{nk}/k - \fracc{2} \big\vert > \epsilon \Big] \leq \epsilon^{-2}\Big( C_1 \sqrt{\frac{n}{k}} + (1 + C_2) \frac{n}{k}\Big)\,.
    \end{equation*}
\end{proof}

Below, we prove \Cref{lem:cov_full} and other supporting lemmas. First, we establish basic statistics about a point $\vec{z_1}$ drawn uniformly at random from an $n$-dimensional ball. The mean is $\vec{0}$ by symmetry of the ball.
\begin{lemma}\label{lem:variance_full}
    The covariance matrix for $\vec{z_1}$ is $\fracc{n+2}I_n$.
\end{lemma}
\begin{proof}
    First, we note that $\E[z_{1i}^2]$ is independent of $i$. Let $r = \norm{\vec{z_1}}$ be the radial coordinate of $\vec{z_1}$. For any $i$, we have
    \begin{equation*}
        n \E[z_{1i}^2] = \sum_{j=1}^n \E[z_{1j}^2] = \E[r^2] = \int_0^1 r^2\phi(r)dr
    \end{equation*}
    where $\phi(r)$ is the pdf for $r$ in an $n$-dimensional ball. We have that $\phi(r)$ is proportional to the volume of an $(n-1)$-dimensional sphere (not ball) of radius $r$, so $\phi(r) = Ar^{n-1}$ for a constant $A$. Because $\int_0^1 \phi(r)dr = 1$, we must have $\phi(r) = n r^{n-1}$. Therefore,
    \begin{equation*}
        \var(z_{1i}) = \E[z_{1i}^2] = \fracc{n} \int_0^1 r^2\phi(r)dr = \int_0^1  r^{n+1}dr = \fracc{n+2}\,.
    \end{equation*}
    By symmetry of the ball,
    \begin{equation*}
        \E[z_{1i}z_{1j}] = \E[-z_{1i} z_{1j}] = -\E[z_{1i}z_{1j}]
    \end{equation*}
    for $i \neq j$, so $\cov(z_{1i}, z_{1j}) = \E[z_{1i}z_{1j}] = 0$.
\end{proof}

We will use the following lemma in an application of the Berry-Esseen theorem to bound the probability $Q_{nk}$.
\begin{lemma}\label{lem:third_moment_full}
    The third moment of $z_{1i}$ is
    \begin{equation*}
        \E[|z_{1i}|^3] = \frac{4v_{n-1}}{v_n(n+3)(n+1)}\,,
    \end{equation*}
    where $v_n$ is the volume of a unit $n$-dimensional ball.
\end{lemma}
\begin{proof}
    The slice of the ball where $z_{1i} = t \in [-1, 1]$ is an $(n-1)$-dimensional ball of radius $\sqrt{1 - t^2}$, which has volume $v_{n-1}(1-t^2)^{\frac{n-1}{2}}$. Therefore, the pdf is
    \begin{equation*}
        f(t) = A^{-1}v_{n-1}(1-t^2)^{\frac{n-1}{2}}\,,
    \end{equation*}
    where $A = \int_{-1}^1 v_{n-1}(1-t^2)^{\frac{n-1}{2}} = v_n$. The third moment is
    \begin{equation*}
        \E[|z_{1i}|^3] = 2 \int_0^1 z_{1i}^3 f(t)dt = \frac{2v_{n-1}}{v_n} \int_0^1 t^3 (1-t^2)^{\frac{n-1}{2}}dt\,.
    \end{equation*}
    Substituting $s = t^2$,
    \begin{equation*}
        \E[|z_{1i}|^3] = \frac{v_{n-1}}{v_n} B\Big(2, \frac{n+1}{2}\Big) = \frac{v_{n-1}}{v_n} \frac{\Gamma(2)\Gamma(\frac{n+1}{2})}{\Gamma(2 + \frac{n+1}{2})} 
        = \frac{v_{n-1}}{v_n} \frac{\Gamma(\frac{n+1}{2})}{\frac{n+3}{2}\frac{n+1}{2}\Gamma(\frac{n+1}{2})}
        = \frac{4v_{n-1}}{v_n(n+3)(n+1)}\,.
    \end{equation*}
\end{proof}

\begin{lemma}\label{lem:prob_full}
For all dimensions $n$ and neighborhood sizes $k$,
\begin{equation*}
    \fracc{2} \leq Q_{nk} \leq \fracc{2} + \fracc{\sqrt{2\pi}} \sqrt{\frac{n+2}{k-1}} + \frac{8C_{BE}}{\sqrt{2\pi(k-1)}}\,.
\end{equation*}
where $C_{BE}$ is a universal constant from the Berry-Esseen theorem.
\end{lemma}
\begin{proof}

    Note that
    \begin{equation*}
        Q_{nk} = \p\Bigg[ \sum_{j > 1} \vec{z_j} \cdot \vec{z_1} \geq - \norm{\vec{z_1}^2}\Bigg]\,.
    \end{equation*}
     Therefore,
    \begin{equation*}
        Q_{nk} \geq \p\Bigg[ \sum_{j > 1} \vec{z_j} \cdot \vec{z_1} \geq 0 \Bigg] = \fracc{2}\,.
    \end{equation*}
    because $\norm{\vec{z_1}}^2 \geq 0$. The last equality holds because $\sum_{j > 1} \vec{z_j}$ and $\vec{z_1}$ are independent vectors in the $n$-ball that are equally likely to have angle $\theta \in [0, \pi/2]$ and angle $\pi - \theta \in [\pi/2, \pi]$.
    
    Next, we prove the upper bound on $Q_{nk}$. Conditioning on $\vec{z_1} = \vec{u}$, we have
    \begin{equation}\label{eq:Qnk_prob_1}
        Q_{nk} = \E_{\vec{z_1}}\Bigg[ \p\Big[ \sum_{j > 1} \vec{z_j} \cdot \vec{u} \geq -\norm{\vec{u}}^2  \mid \vec{z_1} = \vec{u} \Big]\Bigg]
    \end{equation}
    The sum $\sum_{j > 1} \vec{z_j} \cdot \vec{u}$ is a sum of $k-1$ iid random variables with mean $0$ because $\E[\vec{z_j}] = 0$. We will apply the Berry-Esseen theorem to approximate it by a normal random variable; as preliminaries, we need its variance and third moment of each term in the sum. By \Cref{lem:variance_full}, we have $\var(\sum_{j > 1} \vec{z_j} \cdot \vec{u}) = \frac{k-1}{n+2} \norm{\vec{u}}^2$. To calculate the third moment $\E[|\vec{z_j} \cdot \vec{u}|^3]$, note that
    \begin{equation*}
        \E\Big[|\vec{z_j} \cdot \vec{u}|^3\Big] = \E\Big[|\vec{z_j} \cdot \hat{u}|^3\Big]  \, \norm{\vec{u}}^3\,,
    \end{equation*}
    where $\hat{u} = \vec{u}/\norm{\vec{u}}$. By rotational symmetry we may assume that $\hat{u} = (1, 0, \ldots, 0)$. Therefore, $\vec{z_j} \cdot \hat{u} = z_{j1}$ and so $\E[|\vec{z_j} \cdot \vec{u}|^3] = \E[|z_{j1}|^3]\norm{\vec{u}}^3$. By \Cref{lem:third_moment_full},
    \begin{equation*}
        \E\Big[|\vec{z_j} \cdot \vec{u}|^3\Big] = \frac{4v_{n-1}}{v_n(n+1)(n+3)} \norm{\vec{u}}^3 \leq \frac{4v_{n-1}}{v_n(n+1)(n+3)}\,.
    \end{equation*}
    
    By the Berry-Esseen theorem,
    \begin{equation*}
        \Bigg\vert \p\Big[ \sum_{j > 1} \vec{z_j} \cdot \vec{u} \geq -\norm{\vec{u}}^2\Big] - \Phi\Big(\sqrt{\frac{n+2}{k-1}}\norm{\vec{u}}\Big) \Bigg\vert \leq \frac{4C_{\text{BE}}v_{n-1}}{v_n} \frac{(n+2)^{3/2}}{(n+1)(n+3)\sqrt{k-1}}
    \end{equation*}
    where $C_{BE}$ is a universal constant. Finally, we upper bound $\Phi\Big(\sqrt{\frac{n+2}{k-1}}\norm{\vec{u}}\Big)$. The cdf of a normal variable decreases monotonically after $0$, so
    \begin{equation*}
        \Phi\Big(\sqrt{\frac{n+2}{k-1}}\norm{\vec{u}}\Big)  \leq \Phi\Big(\sqrt{\frac{n+2}{k-1}}\Big)= \fracc{2} + \int_0^{\sqrt{\frac{n+2}{k-1}}} \fracc{\sqrt{2\pi}} e^{-t^2/2} dt \leq \fracc{2} + \fracc{\sqrt{2\pi}} \sqrt{\frac{n+2}{k-1}}\,.
    \end{equation*}
    All together,
    \begin{equation*}
        \p\Big[ \sum_{j > 1} \vec{z_j} \cdot \vec{u} \geq -\norm{\vec{u}}^2\Big] \leq \fracc{2} + \fracc{\sqrt{2\pi}} \sqrt{\frac{n+2}{k-1}} + \frac{4C_{\text{BE}}v_{n-1}}{v_n} \frac{(n+2)^{3/2}}{(n+1)(n+3)\sqrt{k-1}}
    \end{equation*}
    for all $\vec{u}$. Consequently, the right-hand side is an upper bound for $Q_{nk}$ by \Cref{eq:Qnk_prob_1}. 

    Finally, we upper bound $v_{n-1}/v_n$ to finish the proof. The volume of a unit $n$-ball is $v_n = \pi^{n/2}/\Gamma(n/2 + 1)$. By Gautschi's inequality \citep{gautschi_elementary_1959},
    \begin{equation*}
        \frac{v_{n-1}}{v_n} = \frac{\Gamma(\frac{n}{2} + 1)}{\Gamma(\frac{n - 1}{2} + 1) \sqrt{\pi}} \leq \sqrt{\frac{n + 2}{2\pi}}\,.
    \end{equation*}
    Therefore,
    \begin{equation*}
        \frac{4C_{\text{BE}}v_{n-1}}{v_n} \frac{(n+2)^{3/2}}{(n+1)(n+3)\sqrt{k-1}} \leq \frac{8C_{BE}}{\sqrt{2\pi(k-1)}}\,.
    \end{equation*}
\end{proof}

We will use \Cref{lem:third_moment_bound} and \Cref{lem:f_tau_approximation} below to bound $\cov(X_1, X_2)$ using an application of the multi-dimensional Berry-Esseen theorem.
\begin{lemma}\label{lem:third_moment_bound}
Suppose that $\hat{u_1}, \hat{u_2}$ are unit vectors in $\R^n$ with angle $\theta \neq 0, \pi$. Let $U$ be the matrix whose columns are $\hat{u_1}, \hat{u_2}$ and let $R = \begin{bmatrix}1 & \cos(\theta) \\ \cos(\theta) & 1 \end{bmatrix}$. If $\vec{z}$ is sampled uniformly at random from the $n$-ball, then
\begin{equation*}
    \E\Big[\norm{R^{-1/2} U^T \vec{z}}^3\Big] = \E\Big[\norm{P_U \vec{z}}^3\Big] \leq \frac{2}{n+2}\,,
\end{equation*}
where $P_U$ is the orthogonal projection matrix onto $\text{span}(\hat{u_1}, \hat{u_2})$.
\end{lemma}
\begin{proof}
    First, note that $R = U^T U$. Therefore,
    \begin{align*}
        \norm{R^{-1/2} U^T \vec{z}}^2 &= \z^T U (R^{-1/2})^T R^{-1/2} U^T \z\\
        &= \z U R^{-1} U^T \z \\
        &= \z U(U^T U)^{-1} U^T \z \\
        &= \z P_U \z \\
        &= \norm{P_U \z}^2\,.
    \end{align*}
    Because $\norm{P_U \z} \leq \norm{\z} \leq 1$, we have
    \begin{equation}\label{eq:PU_bound}
        \E\Big[\norm{R^{-1/2} U^T \vec{z}}^3\Big] = \E\Big[ \norm{P_U \z}^3\Big] \leq  \E\Big[ \norm{P_U \z}^2\Big]\,.
    \end{equation}
    Because $P_U$ is an orthogonal projection matrix,
    \begin{equation*}
        \E\Big[ \norm{P_U \z}^2\Big] = \E[\z^T P_U \z]\,.
    \end{equation*}
    We have
    \begin{equation*}
        \E[\z^T P_U \z] = \E[ \tr (\z^T P_U \z)] = \E[ \tr (P_U\z \z^T)] = \tr\Big( P_U \E[\z \z^T]\Big)\,.
    \end{equation*}
    For $i \neq j$, we have $\E[ (\z \z^T)_{ij}] = \E[\z_i]\E[\z_j] = 0$ because the components are independent and mean-centered. For $i = j$, we have $\E[z_i^2] = \fracc{n}\sum_r \E[z_r^2] = \fracc{n} \norm{\z}^2$\,. Therefore, $\E[\z \z^T] = \fracc{n}\E[\norm{z}^2] I_n$. By \Cref{lem:cov_full}, we have $\E[\norm{z}^2] = \frac{n}{n+2}$. All together,
    \begin{equation*}
         \E\Big[ \norm{P_U \z}^2\Big] = \fracc{n+2} \tr(P_U)\,.
    \end{equation*}
    Because $P_U$ is an orthogonal projection matrix onto a plane (the span of non-parallel vector $\hat{u}_1, \hat{u_2}$), its eigenvalues are $1, 1, 0, \ldots, 0$, so $\tr(P_U) = 2$. By \Cref{eq:PU_bound},  $\E\Big[\norm{R^{-1/2} U^T \vec{z}}^3\Big] \leq \frac{2}{n+2}$.
\end{proof}

\begin{lemma}\label{lem:f_tau_approximation}
    Let $R = \begin{bmatrix} 1 & \rho \\ \rho & 1\end{bmatrix}$ for some $-1 \leq \rho \leq 1$, and suppose that $\bn \sim N(0, R)$. Define $$f(\tau_1, \tau_2, \rho) = \p[\bni{1} \leq \tau_1, \bni{2} \leq \tau_2] - \Phi(\tau_1)\Phi(\tau_2).$$ 
    Then
    \begin{equation*}
    |f(\tau_1, \tau_2, \rho) + f(\tau_1, \tau_2, -\rho)| \leq \frac{4}{\sqrt{2\pi}} (|\tau_1| + |\tau_2|)\,.
\end{equation*}
\end{lemma}
\begin{proof}
   We have $f(\tau_1, \tau_2, -\rho) = \p[\bni{1} \leq \tau_1, -\bni{2} \leq \tau_1] - \Phi(\tau_1)\Phi(\tau_2)$ because $\bni{1}, -\bni{2}$ are normally distributed with variance $1$ and correlation $-\rho$. Therefore,
    \begin{align*}
        f(0, 0, \rho) + f(0, 0, -\rho) &= \p[\bni{1} \leq 0, \bni{2} \leq 0] + \p[\bni{1} \leq 0, -\bni{2} \leq 0] - 2\Phi(0)\Phi(0) \\
        &= \p[\bni{1} \leq 0 ] - \frac{1}{2}\\
        &= 0\,.
    \end{align*}
    Therefore, for nonzero $\tau_1, \tau_2$, 
    \begin{align}
        &|f(\tau_1, \tau_2, \rho) + f(\tau_1, \tau_2, -\rho)|\notag \\
        &\qquad\leq |f(\tau_1, \tau_2, \rho) - f(0, 0, \rho)| + |f(\tau_1, \tau_2, \rho) - f(0, 0, -\rho)|+ |f(0, 0, \rho) - f(0, 0, -\rho)| \notag \\
        &\qquad= |f(\tau_1, \tau_2, \rho) - f(0, 0, \rho)| + |f(\tau_1, \tau_2, \rho) - f(0, 0, -\rho)| \label{eq:first_tau_bound}
    \end{align}
    Therefore, it suffices to show that $|\partial f/\partial \tau_i| \leq \frac{2}{\sqrt{2\pi}}$. Let $g(\tau_1, \tau_2, \rho) = \p[\bni{1} \leq \tau_1, \bni{2} \leq \tau_2]$, where $\rho$ is the correlation of $\bni{1}, \bni{2}$. We will calculate $\partial f/\partial \tau_i$ for $i = 1$; the case $i = 2$ is symmetric. We have
    \begin{equation*}
        \partial f/\partial \tau_1 = \frac{\partial g}{\partial \tau_1} - \varphi(\tau_1) \Phi(\tau_2)\,.
    \end{equation*}
    Letting $\varphi_{\rho}(\bnilower{1}, \bnilower{2})$ denote the joint pdf for $\bni{1}, \bni{2}$, we have
    \begin{equation*}
        \frac{\partial g}{\partial \tau_1} = \frac{\partial g}{\partial \tau_1} \int_{-\infty}^{\tau_1} \int_{-\infty}^{\tau_2} \varphi_{\rho}(\bnilower{1}, \bnilower{2}) d\bnilower{2} d\bnilower{1} = \int_{-\infty}^{\tau_2} \varphi_{\rho}(\tau_1, \bnilower{2})d\bnilower{2} = \varphi(\tau_1) \p[\bni{2} \leq \tau_2 \mid \bni{1} = \tau_1]\,.
    \end{equation*}
    To calculate the conditional CDF, we note that $\varphi_{\rho}$ factors as
    \begin{align*}
        \varphi_{\rho}(\bnilower{1}, \bnilower{2}) &= \frac{1}{2\pi \sqrt{1 - \rho^2}}\exp\Big( - \frac{1}{2(1-\rho^2)}[\bnilower{1}^2 + \bnilower{2}^2 - 2\rho \bnilower{1} z-2]\Big) \\
        &= \frac{1}{2\pi \sqrt{1 - \rho^2}} \exp \Big( -\fracc{2(1-\rho^2)} [(1-\rho^2)\bnilower{1}^2 + (\bnilower{2} - \rho \bnilower{1})^2] \\
        &= \fracc{\sqrt{2\pi}} \exp\Big(- \frac{\bnilower{1}^2}{2}\Big) \fracc{\sqrt{2\pi (1 - \rho^2)}} \exp\Big( -\fracc{2(1-\rho^2)} (\bnilower{2} - \rho \bnilower{1})^2\Big) \\
        &= \varphi(\bnilower{1}) \fracc{\sqrt{2\pi (1 - \rho^2)}} \exp\Big( -\fracc{2(1-\rho^2)} (\bnilower{2} - \rho \bnilower{1})^2\Big)\,,
    \end{align*}
    so the pdf for $\bni{2} \mid \bni{1} = \tau_1$ is the pdf of a normal random variable with mean $\rho \bnilower{1}$ and variance $1 - \rho^2$. Therefore,
    \begin{equation*}
        \p[\bni{2} \leq \tau_2 \mid \bni{1} = \tau_1] = \Phi\Big( \frac{\tau_2 - \rho\tau_1}{\sqrt{1 - \rho^2}}\Big)\,.
    \end{equation*}
    Because $0 \leq \Phi(x) \leq 1$ for any $x$,
    \begin{equation}\label{eq:derivative_bound}
        |\partial f/\partial \tau_1| = \varphi(\tau_1)\Big\vert \Phi\Big( \frac{\tau_2 - \rho\tau_1}{\sqrt{1 - \rho^2}}\Big) - \Phi(\tau_2)\Big\vert \leq \varphi(\tau_1) = \fracc{\sqrt{2\pi}} e^{-\tau_1^2/2} \leq \fracc{\sqrt{2\pi}}\,.
    \end{equation}
   The same upper bound holds for $|\partial f/\partial \tau_2|$. Combining \Cref{eq:first_tau_bound} and \Cref{eq:derivative_bound} completes the proof.
\end{proof}

\begin{lemma}\label{lem:cov_full}
    We have $\var(X_i) = Q_{nk}(1 - Q_{nk})$ for all $i$, and there are constants $C_1, C_2$ such that
    \begin{equation*}
        \cov(X_1, X_2) \leq C_1 \sqrt{\frac{n}{k}} + C_2 \frac{n}{k}
    \end{equation*}
    for all $k \geq 3$.
\end{lemma}
\begin{proof}
    For the variance, observe that $X_i$ is a Boolean random variable with probability of success $Q_{nk}$. For the covariance, we have 
    \begin{equation}\label{eq:cov_first}
        \cov(X_1, X_2) = \E[X_1X_2] - Q_{nk}^2 = \p[X_1 = 1 \text{ and } X_2 = 1] - Q_{nk}^2\,.
    \end{equation}
    By \Cref{lem:prob_full}, there are constants $A_1, A_2$ such that
    \begin{equation}\label{eq:Qnk_bound}
    \Big\vert Q_{nk}^2  -\fracc{4}\Big\vert \leq A_1 \sqrt{\frac{n}{k}} + A_2 \frac{n}{k}\,.
    \end{equation}
    for all $k \geq 2$.
     
    The focus of the rest of the proof is to upper bound $|\p[X_1 = 1 \text{ and } X_2 = 1] - \fracc{4}|$. Let $\hat{v} = \sum_{j > 2} \vec{z_j}$. For $i = 1, 2$, we have
    \begin{equation*}
        \sum_j \vec{z_j} \cdot \vec{z_i} = \hat{v} \cdot \vec{z_i} + \norm{\vec{z_i}}^2 + \vec{z_1} \cdot \vec{z_2}\,,
    \end{equation*}
    so $X_i = 1$ if and only if $\hat{v} \cdot \vec{z_i} \geq -\norm{\vec{z_i}}^2 - \vec{z_1} \cdot \vec{z_2}$. Conditioning on $\vec{z_1} = \vec{u_1}$ and $\vec{z_2} = \vec{u_2}$, we get
    \begin{equation*}
        \p[X_1 = 1 \text{ and } X_2 = 1] = \E_{\z_1, \z_2}[ \p[X_1 = 1 \text{ and } X_2 = 1 \mid \vec{z_1} = \vec{u_1},\, \vec{z_2} = \vec{u_2}]]\,,
    \end{equation*}
    where the conditional probability is
\begin{equation}\label{eq:conditioning}
        \p[X_1 = 1 \text{ and } X_2 = 1 \mid \vec{z_1} = \vec{u_1},\, \vec{z_2} = \vec{u_2}] = \p[ \hat{v} \cdot \vec{u_1} \geq -a_1 \text{ and } \hat{v} \cdot \vec{u_2} \geq -a_2]\,,
    \end{equation}
    where $a_i = \norm{\vec{u_i}}^2 - \vec{u_1} \cdot \vec{u_2}$ for $i = 1, 2$.

    Let $Y_i = \hat{v} \cdot \vec{u_i}$ for $i = 1, 2$, and consider the 2D vector $\vec{Y} = (Y_1, Y_2)$. The right-hand side of \Cref{eq:conditioning} is $\p[ \vec{Y} \geq (-a_1, -a_2)]$. Therefore,
    \begin{equation*}
        \p[X_1 = 1 \text{ and } X_2 = 1 ] = \E_{\z_1, \z_2}[\p[\vec{Y} \geq (-a_1, -a_2)]]
    \end{equation*}
    
    We have that $\vec{Y}$ is the sum of $k-2$ iid 2D random variables:
    \begin{equation*}
        \vec{Y} = \sum_{j > 2} (\vec{z_j} \cdot \vec{u_1}, \vec{z_j} \cdot \vec{u_2})\,
    \end{equation*}
    with mean $0$ because $\E[\vec{z_j}] = 0$ by symmetry of the ball. Using the fact that the variance of the $i$th component of $\vec{z_j}$ is $\fracc{n+2}$ by \Cref{lem:variance_full}, we get that the covariance matrix for $\vec{Y}$ is
    \begin{equation*}
        C_{\vec{Y}} = \frac{k-2}{n+2} \begin{bmatrix}
            \norm{u_1}^2 & \vec{u_1} \cdot \vec{u_2} \\
            \vec{u_1} \cdot \vec{u_2} & \norm{u_2}^2
        \end{bmatrix}\,.
    \end{equation*}
    
    We normalize $\vec{Y}$ so that each component has variance $1$. Let $\hat{Y}_i = Y_i/\sigma_{Y_i}$ and let $\tau_i = a_i/\sigma_{Y_i}$. The main probability of interest is
    \begin{equation*}
        \p[ \vec{Y} \geq (-a_1, -a_2)] = \p[\hat{Y} \geq (-\tau_1, -\tau_2)]\,,
    \end{equation*} 
    where 
    \begin{equation}\label{eq:tau_bound}
        |\tau_i| = \sqrt{\frac{n+2}{k-2}}\,\, \Big\vert \frac{\norm{\vec{u_i}^2} - \vec{u_1} \cdot \vec{u_2}}{\norm{\vec{u_i}}}\Big\vert \leq 2 \sqrt{\frac{n+2}{k-2}}
    \end{equation}
    because $\vec{u_1}$ and $\vec{u_2}$ are in the unit ball. The covariance matrix for $\hat{Y}$ is
    \begin{equation*}
        R := \begin{bmatrix}
            1 & \cos(\theta) \\
            \cos(\theta) & 1
        \end{bmatrix}\,,
    \end{equation*}
    where $\theta$ is the angle between $\vec{u_1}$ and $\vec{u_2}$.

    Let $\bn \sim N(\vec{0}, R)$. We will upper bound $\vert \p[ \hat{Y} \geq (-\tau_1, -\tau_2)] - \p[ \bn \geq (-\tau_1, -\tau_2)]\vert $ using the multi-dimensional Berry-Esseen theorem. We can write $\hat{Y}$ as the sum of $(k-2)$ iid random variables $\hat{Y} = \sum_{j=3}^k \vec{w_j}$, where 
    \begin{equation*}
        \vec{w_j} := (\vec{z_j} \cdot \vec{u_1}/\sigma_{Y_1}, \vec{z_j} \cdot \vec{u_2}/\sigma_{Y_2}) = \sqrt{\frac{n+2}{k-2}} (\z_j \cdot \hat{u_1}, \z_j \cdot \hat{u_2}) = \sqrt{\frac{n+2}{k-2}} U^T \z
    \end{equation*}
    where $\hat{u_i} = \vec{u_i}/\norm{u_i}$ and $U$ is the $(n \times 2)$ matrix with columns $\hat{u_1}, \hat{u_2}$. The covariance matrix $R$ is invertible whenever $\theta \neq 0, \pi$. Therefore when $\theta \neq 0, \pi$, the multi-dimensional Berry-Esseen theorem implies that
    \begin{equation*}
       \vert \p[ \hat{Y} \geq (-\tau_1, -\tau_2)] - \p[ \bn \geq (-\tau_1, -\tau_2)]\vert\leq C_{BE} 2^{1/4} \gamma_{k-2}\,,
    \end{equation*}
    where $C_{BE}$ is a universal constant and
    \begin{align*}
        \gamma_{k-2} &= \sum_{j = 3}^k \E\Bigg[ \norm{R^{-1/2} \vec{w_j}}_2^3\Bigg] = (k-2) \E\Bigg[ \norm{R^{-1/2} \vec{w_3}}_2^3\Bigg] = \frac{(n+2)^{3/2}}{\sqrt{k-2}} \E\Big[ \norm{R^{-1/2} U^T \z_3}_2^3\Big]\,.
    \end{align*}
    By \Cref{lem:third_moment_bound}, we have $\gamma_{k-2} \leq \sqrt{\frac{n+2}{k-2}}$ when $\theta \neq 0, \pi$. Therefore,
    \begin{equation*}
        \vert \p[ \hat{Y} \geq (-\tau_1, -\tau_2)] - \p[ \bn \geq (-\tau_1, -\tau_2)]\vert\leq C_{BE} 2^{1/4} \sqrt{\frac{n+2}{k-2}}
    \end{equation*}
    for all fixed $\vec{u_1}, \vec{u_2}$ such that $\theta \neq 0, \pi$. Vectors $\vec{z_1}, \vec{z_2}$ have angle $\theta \neq 0, \pi$ with probability $1$, so
    \begin{equation}\label{eq:BE_error}
        \Big\vert \E_{\z_1, \z_2}\Big[\p[ \hat{Y} \geq (-\tau_1, -\tau_2)] - \p[ \bn \geq (-\tau_1, -\tau_2)]\Big]\Big\vert \leq C_{BE} 2^{1/4} \sqrt{\frac{n+2}{k-2}}\,.
    \end{equation}

    Next, we show that $\Phi(\tau_1)\Phi(\tau_2)$ is a good approximation to $\p[ \bn \geq (-\tau_1, -\tau_2)]$ in the expectation over $\vec{z_1}, \vec{z_2}$. First, note that
    \begin{equation*}
        \p[ \bn \geq (-\tau_1, -\tau_2)] = \p[ \bn \leq (\tau_1, \tau_2)]\,.
    \end{equation*}
    Use the notation $\p_{\theta}[\bn \leq (\tau_1, \tau_2)]$ to denote the probability when the angle between $\vec{u_1}$ and $\vec{u_2}$ is $\theta$. By \Cref{lem:f_tau_approximation},
    \begin{equation*}
        |\p_{\theta}[\bn \leq (\tau_1, \tau_2)] + \p_{\pi - \theta}[\bn \leq (\tau_1, \tau_2)] - 2 \Phi(\tau_1)\Phi(\tau_2)| \leq \frac{4}{\sqrt{2\pi}}(|\tau_1| + |\tau_2|)\,.
    \end{equation*}
    Therefore when we integrate over all $\vec{u_1}, \vec{u_2}$ in the $n$-ball, we get
    \begin{equation*}
        \Big \vert \E_{\z_1, \z_2}\Big[ \p_{\theta}[ \bn\leq (\tau_1, \tau_2)] - \Phi(\tau_1)\Phi(\tau_2) \Big] \Big\vert \leq \frac{2}{\sqrt{2\pi}}  (|\tau_1| + |\tau_2|)\,.
    \end{equation*}
    By \Cref{eq:tau_bound},
    \begin{equation}\label{eq:normal_expectation_bound}
        \Big \vert \E_{\z_1, \z_2}\Big[ \p_{\theta}[ \bn\leq (\tau_1, \tau_2)] - \Phi(\tau_1)\Phi(\tau_2) \Big] \Big\vert \leq \frac{8}{\sqrt{2\pi}} \sqrt{\frac{n+2}{k-2}}
    \end{equation}
    
    Additionally, we note that
    \begin{equation*}
        \frac{1}{2} \leq \Phi(\tau_i) \leq \phi\Big( \sqrt{\frac{n+2}{k-2}}\Big) \leq \frac{1}{2} + \int_0^{\sqrt{\frac{n+2}{k-2}}} \fracc{\sqrt{2\pi}} e^{-t^2/2}dt \leq \fracc{2} + \fracc{\sqrt{2\pi}} \sqrt{\frac{n+2}{k-2}}\,.
    \end{equation*}
    Consequently,
    \begin{equation}\label{eq:phi_bound}
        \Big\vert \Phi(\tau_1) \Phi(\tau_2) - \fracc{4} \Big\vert \leq \fracc{\sqrt{2\pi}} \sqrt{\frac{n+2}{k-2}} + \fracc{2\pi}\Big( \frac{n+2}{k-2} \Big)\,.
    \end{equation}

    By \Cref{eq:cov_first}, \Cref{eq:Qnk_bound}, \Cref{eq:BE_error}, \Cref{eq:normal_expectation_bound}, and \Cref{eq:phi_bound}, there are constants $C_1, C_2$ such that
    \begin{align*}
        |\cov(X_1, X_2)| &\leq \Big\vert \E_{\z_1, \z_2}\Big[ \p[\hat{Y} \geq (-\tau_1, -\tau_2)] - \p[\bn \geq (-\tau_1, -\tau_2)]\Big] \Big\vert \\
        &\qquad + \Big\vert \E_{\z_1, \z_2}\Big[ \p[\bn \geq (-\tau_1, -\tau_2)] - \Phi(\tau_1)\Phi(\tau_2)] \Big\vert \\
        &\qquad + \Big\vert \Phi(\tau_1)\Phi(\tau_2) - \fracc{4}\Big\vert + \Big\vert Q_{nk}^2 - \fracc{4}\Big\vert \\
        &\leq C_1 \sqrt{\frac{n}{k}} + C_2 \frac{n}{k}
    \end{align*}
    for all $k \geq 3$.
    
\end{proof}

\subsection{Proof of \Cref{thm:half} for boundary points}

Next, we consider the case where $\{\vec{z_i}\}_{i=1}^k$ is a set of points sampled from an $n$-dimensional half ball, modeling the boundary score of a point $x$ on the boundary of an $n$-dimensional manifold.

\begin{proof}[Proof of \Cref{thm:half}]
The random variable $1 - \frac{S_{nk}}{k}$ is non-negative. By Markov's inequality,
    \begin{equation*}
        \p\Big[ 1 - \frac{S_{nk}}{k} \geq \epsilon \Big] \leq \frac{\E\big[ 1 - \frac{S_{nk}}{k}\big]}{\epsilon} = \frac{1 - \p[\vec{w} \cdot \vec{z_1} \geq 0]}{\epsilon} = \frac{\p[\vec{w} \cdot \vec{z_1} < 0]}{\epsilon}\,.
    \end{equation*}
    Applying \Cref{lem:prob_half} completes the proof.
\end{proof}

Below, we prove the supporting lemmas, starting with statistics about a point $\vec{z_1}$ sampled uniformly at random from the half ball. Without loss of generality, we assume that it is the half ball where the first coordinate is $\geq 0$.

\begin{notation}
    Let $\vec{\mu}^{(n)} = \E[\vec{z_1}]$ and let $\sigma_i^{(n)}$ denote the standard deviation of the $i$th component $z_{1i}$. Let $F^{(n)}(t)$ denote the cdf for the first coordinate $z_{11}$, and let $f^{(n)}(t)$ denote the pdf. We will denote the volume of a unit $n$-dimensional ball by $v_n$.
\end{notation}

\begin{lemma}\label{lem:pdf}
The pdf $f^{(n)}(t)$ for the first coordinate of $\vec{z_1}$ is
\begin{equation*}
    f^{(n)}(t) = \frac{2v_{n-1}}{v_n} (1 - t^2)^{\frac{n-1}{2}}\,.
\end{equation*}
\end{lemma}
\begin{proof}
    The slice of the half ball where $z_{11} = t$ is an $(n-1)$-dimensional ball of radius $\sqrt{1 - t^2}$, which has volume $v_{n-1}(1-t^2)^{\frac{n-1}{2}}$. Therefore,
    \begin{equation*}
        f^{(n)}(t) = A^{-1} v_{n-1}(1-t^2)^{\frac{n-1}{2}}\,,
    \end{equation*}
    where $A = \int_0^1 v_{n-1} (1-t^2)^{\frac{n-1}{2}}dt$. This is the volume $v_n/2$ of the half ball.
\end{proof}

\begin{lemma}\label{lem:mean}
    The mean $\vec{\mu}^{(n)}$ has components
    \begin{equation*}
        \mu^{(n)}_i = \begin{cases}
            \frac{2v_{n-1}}{v_n (n+1)}\,, & i = 1\,, \\
            0\,, & i > 1 \,.
        \end{cases}
    \end{equation*}
\end{lemma}
\begin{proof}
    By symmetry, $\mu_i^{(n)} = 0$ for $i > 1$. For $i = 1$,
    \begin{equation*}
        \mu_1^{(n)} = 2\frac{v_{n-1}}{v_n} \int_0^1 t (1-t^2)^{\frac{n-1}{2}}
    \end{equation*}
    by \Cref{lem:pdf}. Using the substitution $s = 1-t^2$,
    \begin{equation*}
        \mu_1^{(n)} = \frac{v_{n-1}}{v_n}\int_0^1 s^{\frac{n-1}{2}}ds =\frac{2v_{n-1}}{v_n (n+1)}\,. 
    \end{equation*}
\end{proof}

\begin{lemma}\label{lem:variance} The variance $(\sigma_i^{(n)})^2$ of the $i$th component of $\vec{z_1}$ is
\begin{equation*}
    (\sigma_i^{(n)})^2 = \begin{cases}
        \fracc{n+2} - (\mu_1^{(n)})^2\,, & i = 1\,,\\
        \fracc{n+2}\,, & i > 1\,.
    \end{cases}
\end{equation*}
\end{lemma}
\begin{proof}
    The same proof as \Cref{lem:variance_full} shows that $\E[z_{1i}^2] = \fracc{n+2}$ for all $i$, because the pdf $\phi(r)$ for $\norm{\vec{z_1}}$ is the same in both the full and half ball.
    
    Therefore, the variance is
    \begin{equation*}
        (\sigma_i^{(n)})^2 = \E[z_{1i}^2] - (\mu_i^{(n)})^2 = \begin{cases}
        \fracc{n+2} - (\mu_1^{(n)})^2\,, & i = 1\,,\\
        \fracc{n+2}\,, & i > 1\,.
    \end{cases}
    \end{equation*}
\end{proof}

\begin{lemma}\label{lem:prob_half}
    Let $\{\vec{z_i}\}_{i=1}^k$ be points drawn independently and uniformly at random from an $n$-dimensional half ball where the first coordinate is $\geq 0$, and let $\vec{w} = \frac{\sum_{i} \vec{z_i}}{\norm{\sum_i \vec{z_i}}}$. Then
    \begin{equation*}
        \p[ \vec{w} \cdot \vec{z_1} < 0] \leq \frac{3}{\sqrt[3]{4}} \,\frac{n+1}{\sqrt[3]{k}}\,.
    \end{equation*}
\end{lemma}
\begin{proof}
    Define $\vec{v} = \fracc{k}\sum_j \vec{z_j}$. Note that $\vec{w} \cdot \vec{z_1} < 0$ if and only if $\vec{v} \cdot \vec{z_1} < 0$, so $\p[ \vec{w} \cdot \vec{z_1} < 0] = \p[ \vec{v} \cdot \vec{z_1} < 0]$. First, will show that $\vec{v}$ concentrates near $\vec{\mu}^{(n)}$. Then, we will condition on $\vec{v}$ being near $\vec{\mu}^{(n)}$ to derive an upper bound on $\p[ \vec{v} \cdot \vec{z_1} < 0].$ When $\vec{v}$ is sufficiently close to $\vec{\mu}^{(n)}$, the conditional probability is small.

    First, we bound the $L^{\infty}$ norm of $\vec{v} - \vec{\mu}^{(n)}$. Let $\delta > 0$. For the $\ell$th component,
    \begin{equation*}
        \p\big[| v_{\ell} - \mu_{\ell}^{(n)}| > \delta|\big] = \frac{\fracc{k} (\sigma_{\ell}^{(n)})^2}{\delta^2} \leq \fracc{k(n+2)\delta^2}
    \end{equation*}
    by Chebyschev's inequality and \Cref{lem:variance}. Let $A_{\delta}$ be the event that $\norm{\vec{v} - \vec{\mu}^{(n)}}_{\infty} \leq \delta$. Applying a union bound yields
    \begin{equation*}
        \p[A_{\delta}^c] \leq \sum_{\ell = 1}^n  \p\big[|v_{\ell} - \mu_{\ell}^{(n)}| > \delta|\big] \leq \frac{n}{k(n+2)\delta^2}\,.
    \end{equation*}

    Now we condition on the event $A_{\delta}$. For any $\delta > 0$, we have that
    \begin{align}
        \p[ \vec{v} \cdot \vec{z_1} < 0] 
        &= \p[ \vec{v} \cdot \vec{z_1} < 0 \mid A_{\delta}] \p[A_{\delta}] + \p[ \vec{v} \cdot \vec{z_1} < 0 \mid A_{\delta}^c] \p[A_{\delta}^c] \notag \\
        &\leq \p[ \vec{v} \cdot \vec{z_1} < 0 \mid A_{\delta}] + \p[A_{\delta}^c] \notag \\
        &\leq \p[ \vec{v} \cdot \vec{z_1} < 0 \mid A_{\delta}] + \frac{n}{k(n+2)\delta^2}\,. \label{eq:condition_Adelta}
    \end{align}
    We can write
    \begin{equation*}
        \vec{v} \cdot \vec{z_1} = \sum_{\ell = 1}^n \mu_\ell^{(n)} z_{1\ell} + \sum_{\ell = 1}^n (v_\ell - \mu_\ell^{(n)}) z_{1\ell} = \mu_1^{(n)} z_{i1} +\sum_{\ell = 1}^n (v_\ell - \mu_\ell^{(n)}) z_{1\ell}\,,
    \end{equation*}
    where the last equality follows by \Cref{lem:mean}. If $\norm{\vec{v} - \vec{\mu}^{(n)}}_{\infty} \leq \delta$, then
    \begin{equation*}
       \Big| \sum_{\ell = 1}^n (v_\ell - \mu_\ell^{(n)}) z_{1\ell} \Big| \leq \sum_\ell |v_\ell - \mu_\ell^{(n)}| \cdot |z_{1\ell}| \leq \delta \norm{\vec{z_1}}_1 \leq \delta \sqrt{n} \norm{\vec{z_1}}_2 \leq \delta \sqrt{n}\,.
    \end{equation*}
    Therefore, 
    \begin{equation*}
        \vec{v} \cdot \vec{z_1} \geq \mu_1^{(n)} z_{i1} - \delta \sqrt{n}\,
    \end{equation*}
    when $\norm{\vec{v} - \vec{\mu}^{(n)}}_{\infty} \leq \delta$. This yields
    \begin{equation*}
        \p[ \vec{v} \cdot \vec{z_1} < 0 \mid A_{\delta}] \leq \p[\mu_1^{(n)} z_{i1} \leq \delta \sqrt{n} ] = F^{(n)}\Big(\frac{\delta \sqrt{n}}{\mu_1^{(n)}}\Big)\,.
    \end{equation*}
    Combining with \Cref{eq:condition_Adelta} gives us
    \begin{equation}\label{eq:bound_1}
        \p[ \vec{v} \cdot \vec{z_1} < 0] \leq F^{(n)}\Big(\frac{\delta \sqrt{n}}{\mu_1^{(n)}}\Big) + \frac{n}{k(n+2)\delta^2}\,,
    \end{equation}
    where $F^{(n)}$ is the CDF for the first coordinate of $\vec{z_1}$. The CDF is monotonically increasing and satisfies
    \begin{equation*}
        F(t) = \int_0^t f^{(n)}(s)ds \leq t \max_{s \in [0, 1]} f^{(n)}(s)\,.
    \end{equation*}
    By \Cref{lem:pdf}, the pdf $f^{(n)}(s)$ is monotonically decreasing, so 
    \begin{equation*}
        \max_{s \in [0, 1]} f^{(n)}(s) = f^{(n)}(0) = \frac{2v_{n-1}}{v_n} = (n+1) \mu_1^{(n)}\,.
    \end{equation*}
    Therefore,
    \begin{equation*}
        F^{(n)}\Big(\frac{\delta \sqrt{n}}{\mu_1^{(n)}}\Big) \leq \sqrt{n}(n+1)\delta\,.
    \end{equation*}
    Substituting into \Cref{eq:bound_1} gives us
    \begin{equation}\label{eq:bound_2}
        \p[ \vec{v} \cdot \vec{z_1} < 0] \leq \sqrt{n}(n+1)\delta + \frac{n}{k(n+2)\delta^2} =: g(\delta)
    \end{equation}
    for all $\delta$. 
    
    To finish the proof, we will optimize the choice of $\delta$ as a function of $n$ and $k$. The right-hand side $g(\delta)$ approaches $\infty$ as $\delta \to 0$ or $\delta \to \infty$, so it achieves a minimum. Write $g(\delta) = a\delta + b\delta^{-2}$, where $a = \sqrt{n}(n+1)$ and $b = \frac{n}{k(n+2)}$. The only critical point occurs when
    \begin{equation*}
        \delta_* = \sqrt[3]{\frac{2b}{a}}\,.
    \end{equation*}
    At the critical point, we have
    \begin{equation*}
        g(\delta_*) = \frac{3}{\sqrt[3]{4}} \sqrt[3]{a^2 b} = \frac{3}{\sqrt[3]{4}} \sqrt[3]{\frac{n^2 (n+1)^2}{(n+2)k}} \leq  \frac{3}{\sqrt[3]{4}} \sqrt[3]{\frac{(n+1)^4}{(n+1)k}} = \frac{3}{\sqrt[3]{4}} \,\frac{n+1}{\sqrt[3]{k}}\,.
    \end{equation*}
    By \Cref{eq:bound_2},
    \begin{equation*}
        \p[ \vec{v} \cdot \vec{z_1} < 0] \leq \frac{3}{\sqrt[3]{4}} \,\frac{n+1}{\sqrt[3]{k}}\,.
    \end{equation*}
\end{proof}

\section{Methodological details for experiments}\label{sec:experiment_details}

\subsection{Intrinsic dimension estimation}

HP requires local dimension estimation (Step 3 of \Cref{alg:boundary_hyperplane}). Unless otherwise stated, we use a spectral gap threshold $\tau$ to determine the dimension. Given a neighborhood $\{x_{ij}\}_j$ for point $x_i$, let $\lambda_1 \geq \lambda_2 \geq \cdots \geq \lambda_D$ be the eigenvalues of the covariance matrix (for the local neighborhood) in descending order. The local dimension estimate at $x_i$ is the minimum index $m$ such that $\lambda_{m}/\lambda_{m+1} > \tau$. If no such index exists, we use dimension $D$. The default threshold is $\tau = 7$.

\subsection{AUC values for synthetic data experiments}\label{sec:auc}
In \Cref{sec:boundary_experiments}, we calculated AUC values for interior vs. boundary points of disks (and $n$-dimensional balls) embedded in $\R^D$. Because the ground-truth boundary of a disk or ball has measure $0$, there are not any points on the ground-truth disk (or $n$-ball) boundary, so we defined the ``empirical'' boundary to be the set of points with norm larger than a threshold $\tau_B$ and the ``empirical'' interior to be the set of points with norm smaller than a threshold $\tau_I$. We used $\tau_B = .95$ and $\tau_I = .8$ in \Cref{sec:boundary_experiments}. For the disks embedded in $\R^D$, the norm was calculated with respect to the first two coordinates (the plane that contains the disk), and for the $n$-dimensional balls in $\R^n$, we used the $n$-dimensional Euclidean norm. In this section, we present the AUC values for different choices of $\tau_B \in \{.91, \ldots, .98\}$. Our conclusions do not qualitatively change when we use a different threshold $\tau_B$.

In \Cref{fig:disks_high_ambient}, we considered the AUC values for 2D disks embedded in $\mathbb{R}^D$ with increasing ambient dimension $D$, at fixed boundary threshold $\tau_B = .95$. In \Cref{fig:disks_high_ambient_AUC_sweep}, we show the AUC values for varying thresholds $\tau_B \in \{.91, \ldots, .98\}$. At each threshold $\tau_B$, the DU method performs best, followed by HP, followed by LoDD.

\begin{figure}
    \centering
    \includegraphics[width = \textwidth]{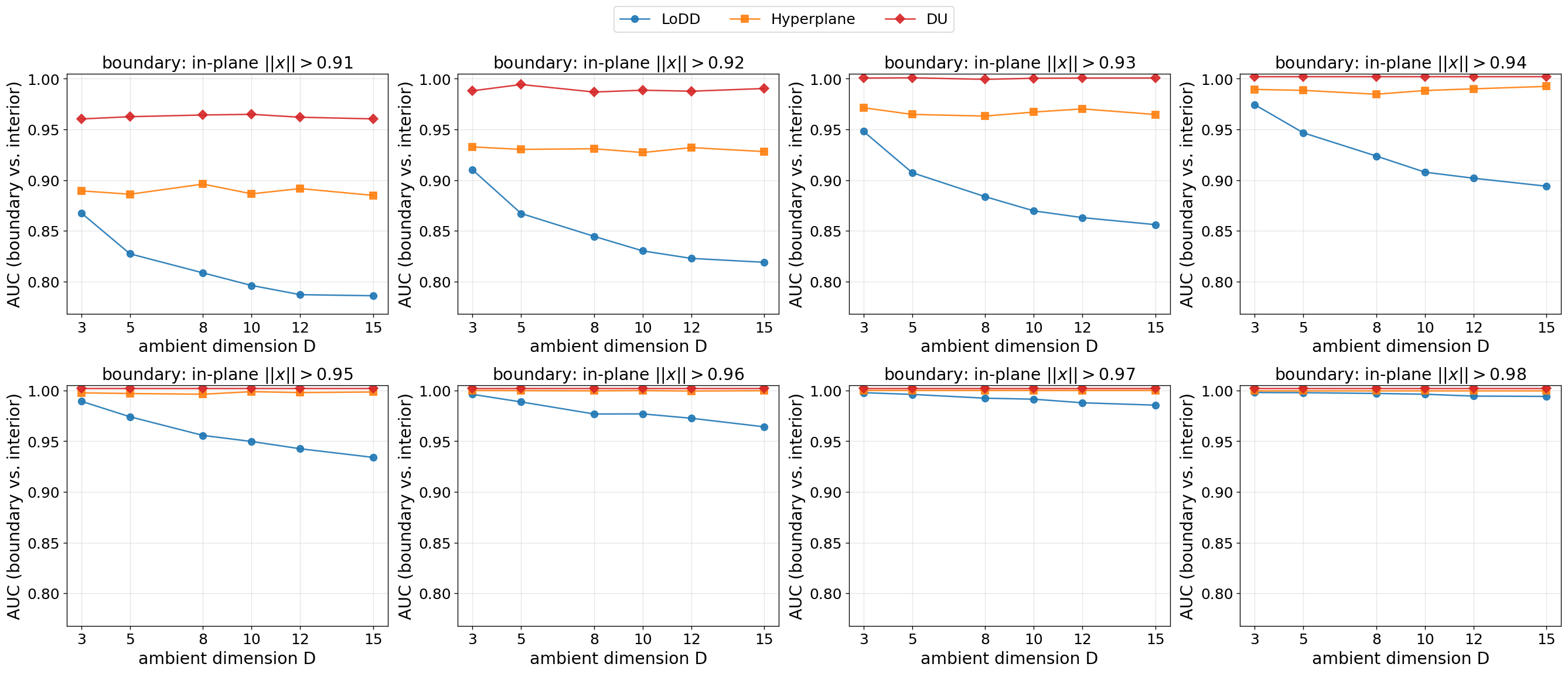}
    \caption{Point clouds are sampled from 2D disks embedded in $\mathbb{R}^{D}$ for $D = 3, 5, 8, 10, 12, 15$ with Gaussian noise level $\sigma = .01$. Each point cloud has $N = 10,000$ points, and we use neighborhood size $k = 100$. Interior points are defined to be points with norm less than $.8$, and boundary points are defined to be points with norm larger than $\tau_B$, which increases from $\tau_B = .91$ to $\tau_B = .98$ across the panels. In each panel, we plot the AUC values as a function of ambient dimension $D$ for each of the three methods.}
\label{fig:disks_high_ambient_AUC_sweep}
\end{figure}

In \Cref{fig:disks_increasing_sigma}(B), we calculated AUC values for 2D disks embedded in $\R^3$ with increasing noise $\sigma \in [.01, .05]$, at fixed boundary threshold $\tau_B = .95$. In \Cref{fig:disks_increasing_sigma_AUC_sweep}, we show the AUC values for varying thresholds $\tau_B \in \{.91, \ldots, .98\}$. For $\tau_B \geq .95$, the relative performance of each method is the same as when $\tau = .95$, where the hyperplane method with fixed dimension does best, closely followed by DU. For $.91 \leq \tau_B \leq .94$, DU outperforms every method at the lowest noise levels $\sigma = .01$ and $\sigma = .02$.

\begin{figure}
    \centering
    \includegraphics[width = \textwidth]{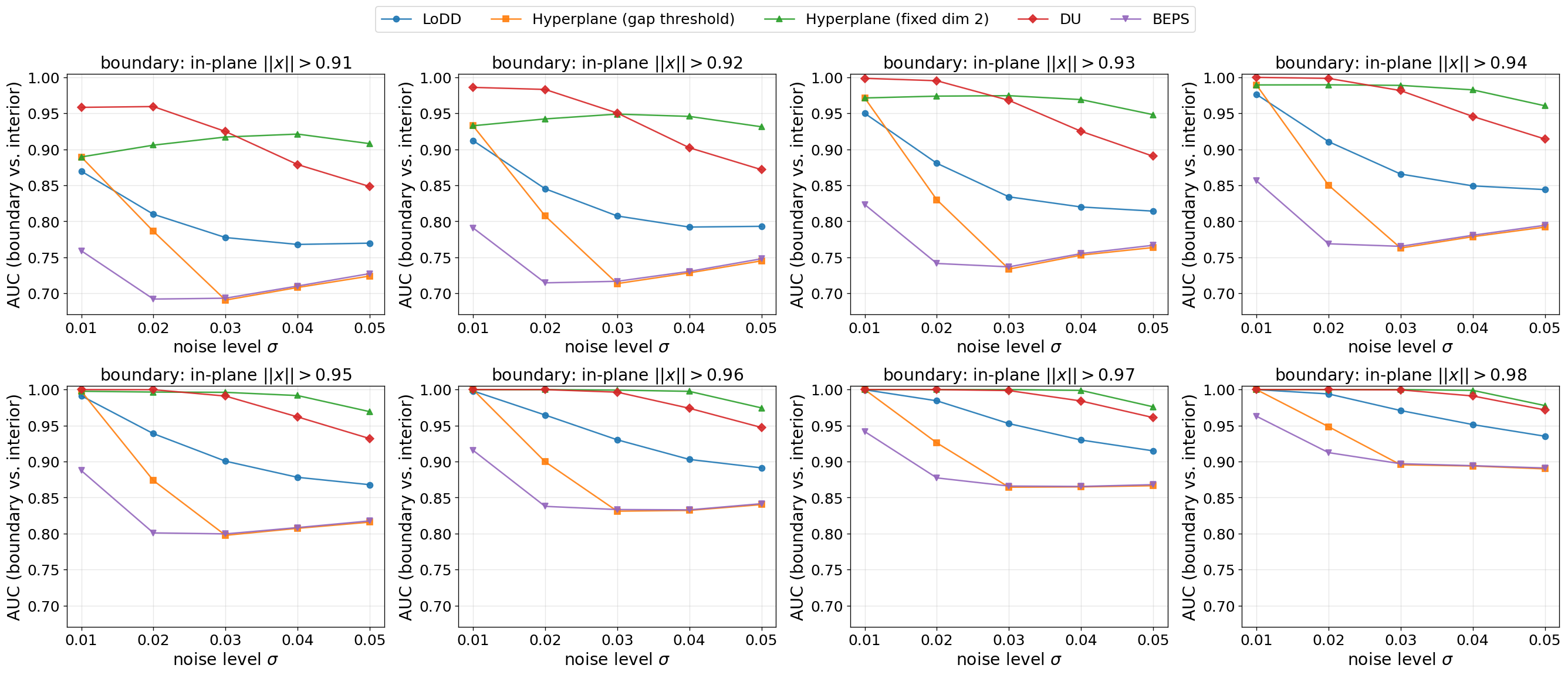}
    \caption{At each noise level $\sigma$, we sample a point cloud from a 2D disk in $\R^3$ with isotropic Gaussian noise. Each point cloud has $N = 10,000$ points, and we use neighborhood size $k = 100$. Interior points are defined to be points with norm less than $.8$, and boundary points are defined to be points with norm larger than $\tau_B$, which increases from $\tau_B = .91$ to $\tau_B = .98$ across the panels. In each panel, we plot the AUC values as a function of noise level $\sigma$ for each of the three methods.}
    \label{fig:disks_increasing_sigma_AUC_sweep}
\end{figure}

In \Cref{fig:disks_increasing_sigma}(C), we calculated AUC values for 2D disks embedded in $\R^D$ with increasing dimension $D$ and fixed noise level $\sigma = .05$, at fixed boundary threshold $\tau_B = .95$. In \Cref{fig:disks_increasing_D_sig.05_AUC_sweep}, we show the AUC values for varying thresholds $\tau_B \in \{.91, \ldots, .98\}$. At all threshold $\tau_B$, the hyperplane method with fixed dimension performs best, followed by DU, followed for LoDD, followed by BEPS and the hyperplane method with a spectral gap threshold condition.

\begin{figure}
    \centering
    \includegraphics[width = \textwidth]{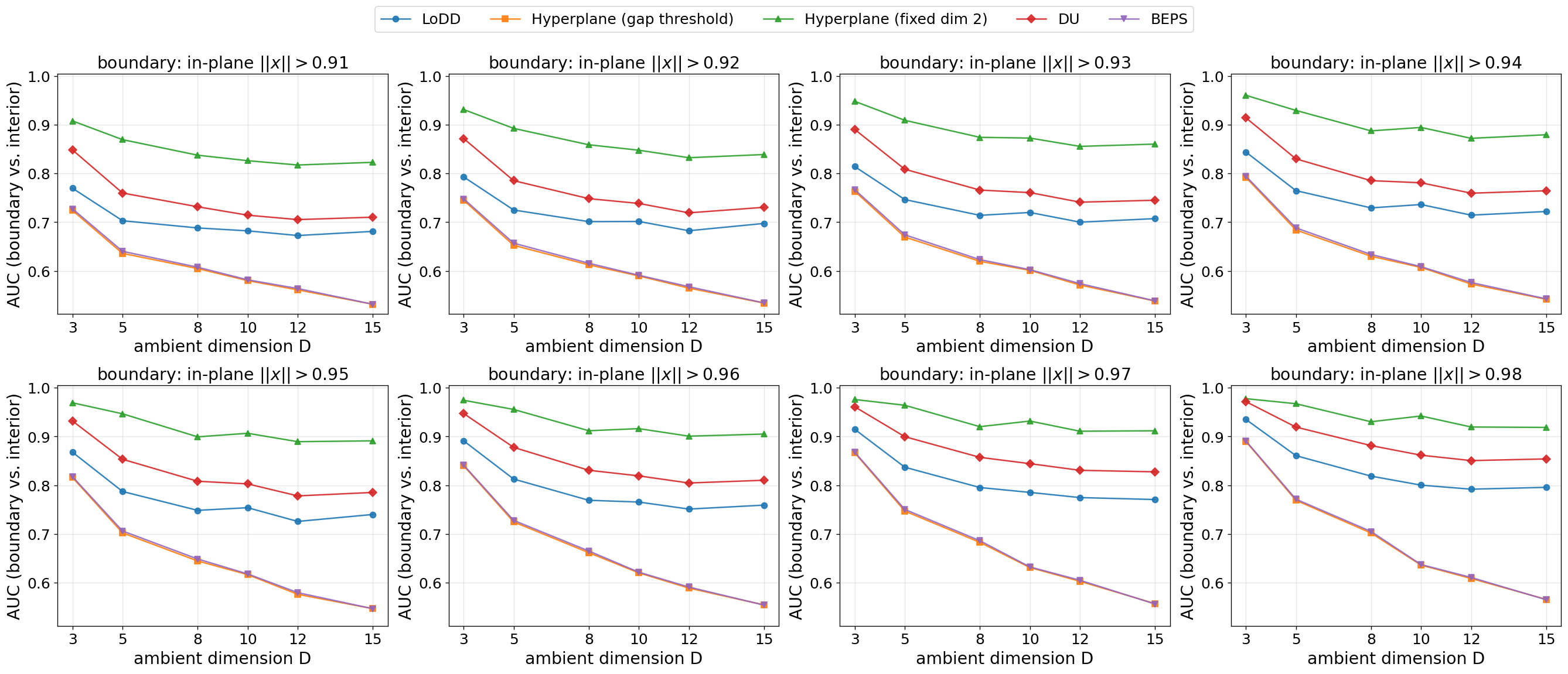}
    \caption{At each ambient dimension $D$, we sample a point cloud from a 2D disk in $\R^D$ with isotropic Gaussian noise $\sigma = .05$. Each point cloud has $N = 10,000$ points, and we use neighborhood size $k = 100$. Interior points are defined to be points with norm less than $.8$, and boundary points are defined to be points with norm larger than $\tau_B$, which increases from $\tau_B = .91$ to $\tau_B = .98$ across the panels. In each panel, we plot the AUC values as a function of ambient dimension $D$ for each of the four methods.}
    \label{fig:disks_increasing_D_sig.05_AUC_sweep}
\end{figure}

\FloatBarrier
\subsection{Spatial Transcriptomics}\label{sec:ST_details}

\subsubsection{Data preprocessing}

The raw data has $741$ spatial spots and $27,984$ genes. 

\paragraph{Quality control.} 
We do gene quality control by discarding genes that are detected at fewer than $3$ spots. This results in the removal of $12,372$ genes.

\paragraph{Library-size normalization.} We normalize each spot to a common total count of $10,000$. More precisely, if $x_{ig}$ is the number of transcripts of gene $g$ at spot $i$, the library size at spot $i$ is $L_i = \sum_g x_{ig}$. We rescale the counts $x_{ig} := (x_{ig}/L_i)\times 10,000$ so that each spot has library size $10,000$.

\paragraph{Log transform.} For each spot $i$ and gene $g$, we apply a log transform $x_{ig} := \log(1 + x_{ig})$, where $x_{ig}$ is the output from the library-size normalization step.

\paragraph{Highly variable genes.} 
We select the $3,000$ most highly-variable genes, discarding the rest.

\paragraph{Scaling.} We normalize each gene so that it has mean $0$ and variance $1$ by mapping each $x_{ig}$ to its $z$-score. More precisely, if $\mu_g$ and $\sigma_g$ are the mean and standard deviation for gene $g$ expression across spots, we calculate $x_{ig} := (x_{ig} - \mu_g)/\sigma_g$. We clip values at max $z$-score $10$, i.e., $x_{ig} := \min(x_{ig}, 10)$, so that no single gene dominates the PCA step below.

\paragraph{PCA dimension reduction.} Let $x_i = (x_{ig})_g$ be the gene expression values at spot $i$ after the preprocessing steps above. We use PCA to reduce $\{x_i\}_i$ to $30$ dimensions.

\subsubsection{Neighborhood calculation for the boundary detection algorithms}

Both HP and DU require a local neighborhood at each point (see lines 2 of \Cref{alg:boundary_hyperplane} and \Cref{alg:boundary_neighborhood_uniformity}, respectively). We define the neighborhoods as follows.

For each spot, we concatenate the $30$-dimensional PCA gene expression coordinates with the $2$D spatial coordinates, yielding a $32$-dimensional vector at each spot. The gene and spatial coordinates have very different scales, so to determine the nearest neighbors of each spot, it does not make sense to naively calculate Euclidean distance in $\R^{32}$. 

Instead, we consider spatial and gene distance independently. For integer $k'$, let $N^g(k', i)$ denote the nearest $k'$ neighbors of the $i$th spot with respect to gene coordinates (Euclidean distance in $\R^{30}$), and let $N^s(k', i)$ denote the nearest $k'$ neighbors of the $i$th spot with respect to spatial coordinates (Euclidean distance in $\R^2$). At neighborhood size parameter $k$, we define the neighborhood of spot $i$ to be the intersection
\begin{equation*}
    N^g(k', i) \cap N^s(k', i)
\end{equation*}
for the minimum integer $k'$ such that the intersection has size $k$.

\subsection{LLM Activation Geometry}\label{sec:llm_details}

In \Cref{sec:llm}, we applied our boundary detection algorithms to study LLM activation geometry. The results for Pythia-2.8B were shown in \Cref{fig:pythia}, and the results for Llama-3.1-8B were shown in \Cref{fig:llama}. Here, we detail our methodology.

\subsubsection{Data set}

We stream documents from the C4 (English) data set \citep{c4}, keeping those that have at least $400$ characters excluding whitespace. For each document and each $k \in \{15, 20, \ldots, 150\}$, we remove all newlines and then re-insert them every $\leq k$ characters, breaking at the nearest word boundary $\leq k$ so that each line has at most $k$ characters.\footnote{Occasionally, we encounter a token with $\geq k$ characters itself. These tokens are put onto their own line, thus creating a line that is over $k$ characters. These lines are rare.} We continue extracting documents from C4 until we reach at least $250,000$ non-newline tokens in total. This is satisfied by the first $28$ documents. (See \Cref{tab:tokens} for total token counts for Pythia-2.8B and Llama-3.1-8B.) This process yields $28$ documents $\times 28$ widths $=784$ total sequences. We cap each sequence at $512$ tokens, including newlines.

Each non-newline token is labeled by its character position $i$ and line width $w \in \{15, 20, \ldots, 150\}$. The character position of a token is defined to be the number of characters used from the beginning of the line through the end of the token.

\begin{table}[h]
  \centering
  \begin{tabular}{|l|r|r|}
    \hline
     & Pythia-2.8B & Llama-3.1-8B \\
    \hline
    Total tokens         & 271{,}100 & 269{,}651 \\
    Non-newline tokens   & 250{,}437 & 250{,}641 \\
    \hline
  \end{tabular}
  \caption{Number of tokens extracted for each model.}
  \label{tab:tokens}
\end{table}

\subsubsection{Models and layer selection}

We test two models, Pythia-2.8B and Llama-3.1-8B, which have dimensions $D = 2560$ and $D = 4096$, respectively. Both models have $32$ layers. For each token in our data set, we extract the residual-stream activations at layer $4$ in both models. We chose layer $4$ after visual inspection of the averaged character-counting manifold for candidate layers $4$, $11$, $16$, $24$, and $29$.

\subsubsection{Averaged manifold}

For each $(i, w)$, we average the residual-stream vectors at layer $4$ for tokens with character position $i$ and line width $w$. We discard $(i, w)$ if there are fewer than $5$ tokens with position $i$ and line width $w$. This gives us $N = 2,311$ points for the Pythia-2.8B manifold and $N = 2,335$ points for the Llama-3.1-8B, manifold.\footnote{If every token had position $1 \leq i \leq w$, there would be at most $2,310$ possible labels $(i, w)$. For a line of width $w$, every token has position $1 \leq i \leq w$ with the exception of tokens that have length $\geq w$. A token with more than $w$ characters forms its own line. This gives us occasional tokens with label $(i, w)$ for $i > w$, which is why the averaged manifolds for both Pythia-2.8B and Llama-3.1-8B have more than $2,310$ points.}

\subsubsection{Boundary detection on the averaged manifold}

First, we use PCA to reduce the averaged manifold ($N \approx 2300$ for each model) to $30$ dimensions. Then, we run the DU method and the hyperplane method with fixed dimension $2$, both with neighborhood size $k = 100$. (The manifold is intrinsically $2$-dimensional because it is parameterized by character position $i$ and line width $2$.)

The hyperplane method yields a vector $\vec{w}$, in local PCA coordinates, at each point, which we lift to a vector $\vec{w}_{\text{amb}}$ in the ambient $30$-dimensional PCA basis. The DU method gives us a vector $\vec{v_i} \in \R^{30}$ at each point. In \Cref{fig:pythia} and \Cref{fig:llama}, we show the vectors for points whose boundary scores are in the top decile.

\subsubsection{Boundary detection on the un-averaged, per-token manifold}

Instead of averaging the residual-stream activations for each $(i, w)$, we consider the entire point cloud of size size $N \approx 250,000$ (one point per token). Again, we use PCA to reduce to $30$ dimensions. Then, we run the DU method and the hyperplane method with neighborhood size $k = 100$. We use a spectral gap threshold condition (with default threshold $\tau = 7$) for the hyperplane method instead of fixing dimension $2$ because the intrinsic dimension is no longer $2$. At each $(i, w)$, we average the boundary scores for all tokens at position $i$ within a line of width $w$.

\subsubsection{3D visualization}

For both models, the first principal component of the averaged manifold is dominated by a few outliers, so in \Cref{fig:pythia} and \Cref{fig:llama}, we plot in PC2-PC3-PC4 coordinates.

\section{Other boundary detection variants}

Both CO \citep{gravitational} and BEPS \citep{beps} use the vector $\vec{v_i} := \fracc{k} \sum_j \frac{x_{ij} - x_i}{\norm{x_{ij} - x_i}}$ at each point $x_i$, where $\{x_{ij}\}$ are its $k$ nearest neighbors. The CO score at point $x_i$ is 
\begin{equation*}
    k^2  \sum_{j=1}^k m_i^{-1} m_j^{-1}\vec{v_i} \cdot \vec{v_j}\,,
\end{equation*}
where $m_i^{-1} = \sum_j \norm{x_{ij} - x_i}$. The CO score is high on low-density points, regardless of whether they are interior or boundary points, so the CO method struggles with non-uniform density. The difference between our DU method and the CO method is that we removed the term $m_i^{-1} m_j^{-1}$ and normalized the score so that it is always less than or equal to $1$.

We tested several other boundary detection variants (see \Cref{tab:variants}), with results shown in \Cref{fig:other_variants}. Each variant uses either the vector $\vec{v_i}$ or a slightly different vector $\vec{u_i} := \sum_j (x_{ij} - x_i)$ where the neighbor-displacement vectors are not normalized, and then each variant has a score that is constructed from the chosen vector. The cosine variant uses $\vec{v_i}$ but measures the cosine between $\vec{v_i}$ and $\sum_j \vec{v_j}$ instead of taking the dot product; this variant isolates the effect of neighbor alignment between vectors, which was the rationale in \citet{gravitational}. We find that the cosine variant works poorly because the vectors of \emph{all} neighbors tend to be similar due to the fact that their neighborhoods are similar, regardless of whether they are on a boundary or not. The magnitude variant $\norm{\vec{v_i}}$ works almost as well as DU on disks and $n$-balls, but as discussed in \Cref{sec:neighborhood_uniformity}, the DU variant is more robust to noise. The two variants that use $\vec{u_i}$ do not work well with non-uniform density because $\norm{\vec{u_i}}$ is large when the density is low. The DU variant was the best overall performer across varying noise levels, varying dimension, and non-uniform density.

\begin{table}[htbp]
\centering
\begin{tabular}{lll}
\hline
Variant & Vector & Score \\
\hline
 CO \citep{gravitational} &  $\vec{v_i}$ &  $k^2  \sum_{j=1}^k m_i^{-1} m_j^{-1}\vec{v_i} \cdot \vec{v_j}$ \\
 Directional uniformity &  $\vec{v_i}$  &  $\sum_{j=1}^k \vec{v_i} \cdot \vec{v_j} $\\
 CO magnitude &  $\vec{v_i}$  &  $\norm{\vec{v_i}}$\\
 Cosine &  $\vec{v_i}$ &  $\cos(\vec{v_i}, \sum_j \vec{v_j})$\\
Sum displacement &  $\vec{u_i}$ &  $\sum_j \vec{u_i} \cdot \vec{u_j} $\\
 Mean direction &  $\vec{u_i}$ &  $\sum_j \cos(\vec{u_i}, \vec{u_j})$ \\
\hline
\end{tabular}
\caption{Other boundary-detection methods that we tested.}
\label{tab:variants}
\end{table}

\begin{figure}
    \centering
    \subfloat[]{\includegraphics[width = \textwidth]{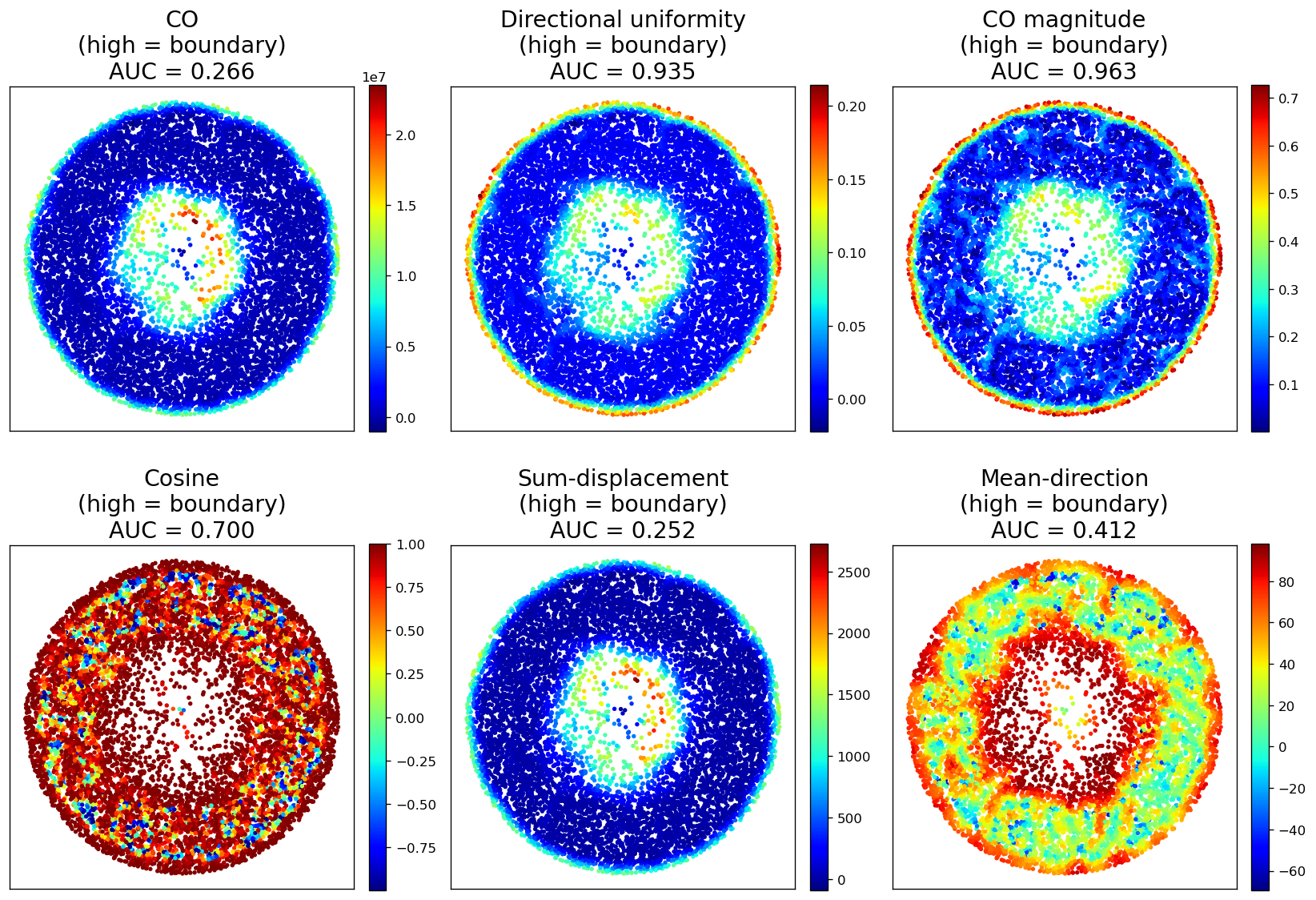}}\\
    \subfloat[]{\includegraphics[width = \textwidth]{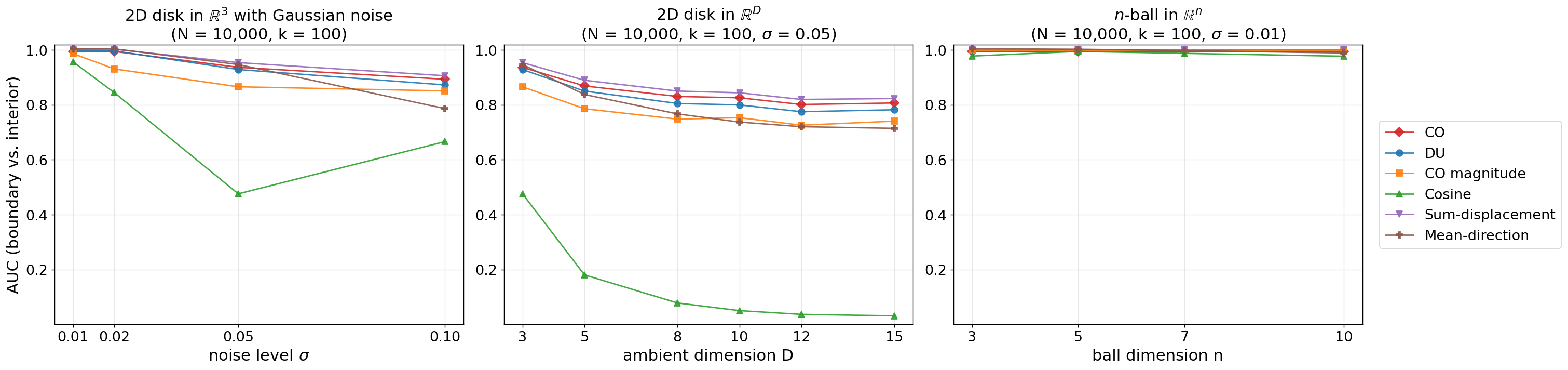}}
    \caption{We compare the variants from \Cref{tab:variants}. (A) Scores on the non-uniformly sampled disk from \Cref{fig:density} with $N = 10,000$ points and neighborhood size $k = 100$. AUROC scores are for boundary points (norm $> .98$) vs deep interior points (norm $< .23$), like in \Cref{fig:density}. AUROC scores for boundary vs. interior points on 2D disks and $n$-balls, where interior and boundary points are defined as having norms $< .8$ and $> .95$, respectively.}
    \label{fig:other_variants}
\end{figure}

\end{document}